\documentclass{article}

\PassOptionsToPackage{numbers, compress}{natbib}

 \usepackage[preprint]{neurips_2026}

\usepackage[utf8]{inputenc} 
\usepackage[T1]{fontenc}    
\usepackage{hyperref}       
\usepackage{url}            
\usepackage{booktabs}       
\usepackage{amsfonts}       
\usepackage{nicefrac}       
\usepackage{microtype}      
\usepackage{xcolor}         
\usepackage{graphicx}       
\usepackage{subcaption}    
\graphicspath{{figs/}}
\usepackage{amsmath}
\usepackage{amssymb}
\usepackage{amsthm}
\usepackage{mathtools}
\usepackage{algorithm}
\usepackage{algorithmic}
\usepackage{cleveref}
\usepackage{tikz}
\usepackage{wrapfig}
\usepackage[textsize=small]{todonotes}
\crefname{assumption}{Assumption}{Assumptions}
\crefname{theorem}{Theorem}{Theorems}
\crefname{proposition}{Proposition}{Propositions}
\crefname{corollary}{Corollary}{Corollaries}
\crefname{lemma}{Lemma}{Lemmas}
\crefname{remark}{Remark}{Remarks}

\newtheorem{theorem}{Theorem}

\newtheorem{proposition}[theorem]{Proposition}

\newtheorem{assumption}{Assumption}
\newtheorem{remark}{Remark}

\newcommand{\E}{\mathbb{E}}
\newcommand{\R}{\mathbb{R}}
\newcommand{\cA}{\mathcal{A}}
\newcommand{\cX}{\mathcal{X}}
\newcommand{\cD}{\mathcal{D}}
\newcommand{\cH}{\mathcal{H}}
\newcommand{\cM}{\mathcal{M}}
\newcommand{\cF}{\mathcal{F}}
\newcommand{\cN}{\mathcal{N}}
\newcommand{\Regret}{\mathrm{Regret}}
\newcommand{\TV}{\mathrm{TV}}
\newcommand{\CRPS}{\mathrm{CRPS}}

\DeclareMathOperator*{\argmax}{arg\,max}

\newcommand{\textproc}[1]{\textnormal{\textsc{#1}}}

\newcommand{\cL}{\mathcal{L}}

\title{PFN-TS: Thompson Sampling for Contextual Bandits via Prior-Data Fitted Networks}

\author{%
  Yan Shuo Tan$^{1}$\thanks{Correspondence: \texttt{yanshuo@nus.edu.sg}} \quad
  Kenyon Ng$^{2}$ \quad
  Ruizhe Deng$^{1}$ \\
  \AND
  Sumetha Loganathan$^{1}$ \quad
  Qiong Zhang$^{3}$ \quad
  Bibhas Chakraborty$^{4}$ \\[0.5em]
    $^{1}$National University of Singapore \quad
    $^{2}$Monash University \\[0.1em]
    $^{3}$Renmin University of China \quad
    $^{4}$Duke-NUS Medical School
}

\begin{document}

\maketitle

\begin{abstract}
Thompson sampling is a widely used strategy for contextual bandits: at each round, it
samples a reward function from a Bayesian posterior and acts greedily under that sample. Prior-data fitted
networks (PFNs), such as TabPFN v2+ and TabICL v2, are attractive candidates for this
purpose because they approximate Bayesian posterior predictive distributions in a single
forward pass. However, PFNs predict noisy future rewards, while Thompson sampling requires
uncertainty over the latent mean reward function. We propose PFN-TS, a Thompson sampling
algorithm that converts PFN posterior predictives into mean-reward samples using a
subsampled predictive central limit theorem. The method estimates posterior variance from a geometric grid
of $O(\log n)$ dataset prefixes rather than the full $O(n)$ predictive sequence used in
previous predictive-sequence approaches, and reuses TabICL's cached representations across
rounds. We prove consistency of the subsampled variance estimator and give a Bayesian
regret bound that decomposes PFN-TS regret into exact posterior-sampling regret under the
PFN prior plus approximation terms. Empirically, PFN-TS achieves the best average rank
across nonlinear synthetic and OpenML classification-to-bandit benchmarks, remains
competitive on linear and BART-generated rewards, and attains the highest estimated policy
value in an offline mobile-health evaluation. Code is available at \url{https://anonymous.4open.science/r/PFN_TS-36ED/}.
\end{abstract}

\section{Introduction}
\label{sec:intro}

Contextual bandits formalize sequential decision-making under uncertainty: at each round
a learner observes a context, selects an action, and receives a stochastic reward, with
the goal of minimizing cumulative regret over $T$ rounds. Thompson Sampling
(TS) \citep{thompson1933likelihood, russo2016information} addresses this by maintaining a
Bayesian posterior over reward functions and sampling actions from it, but its performance
is only as good as the underlying reward model. In tabular applications such as mobile
health \citep{klasnja2015microrandomized, bell2020drinkless} and recommendation systems
\citep{li2010contextual}, linear reward models \citep{agrawal2013thompson} are standard
yet often badly misspecified: human behavior involves interactions, threshold effects, and
heterogeneity that linearity cannot capture. Richer alternatives---kernel methods
\citep{chowdhury2017kernelized}, neural networks \citep{zhang2021neural}, and
tree ensembles \citep{nilsson2024tree, deng2026bfts}---each trade one limitation for
another: kernel methods scale poorly in high dimensions, neural networks carry heavy
computational overhead, and tree-based methods impose piecewise-constant structure that
fits poorly in smooth or near-linear regimes. Compounding these model-specific weaknesses,
all these approaches require hyperparameter choices, such as kernel bandwidth, network architecture, or tree
depth, that are difficult to tune in an online setting.

Recently, tabular foundation models (TFMs) have emerged as state-of-the-art for regression and
classification on small-to-medium tabular datasets.
Among TFMs, TabPFN \citep{hollmann2025tabpfn} and
TabICLv2 \citep{qu2026tabiclv2} stand out because they are implemented as amortized Bayesian engines, also known as Prior-Data Fitted Networks (PFNs) \citep{muller2022transformers}:
They are trained to approximate the Bayesian posterior predictive distributions (PPDs) arising from a broad prior over data-generating processes.
When deployed on a new dataset, they approximate
the PPD via a single forward pass, without any parameter updates, a paradigm
known as in-context learning (ICL). This broad prior and ICL approach confer three
properties that are unusually well-suited to bandit deployment. 
First, \emph{model-fitting is very fast} since it requires no parameter updates.
Second,
\emph{no hyperparameter tuning} is required and yet these models outperform AutoML baselines on benchmark datasets.
Third, the models
exhibit \emph{structural adaptivity}: because the prior covers a wide range of DGPs, predictions automatically reflect
whichever structure (e.g. sparsity, smoothness, nonlinearity, or heteroscedasticity) is consistent with the observed data, avoiding the misspecification
that afflicts single-model-class baselines \citep{zhang2025tabpfnmodelruleall}. 
Unfortunately, while these models output PPDs that quantify uncertainty in a noisy future response, they do not directly output posterior distributions for the latent regression function, which is what is needed by TS.

A potential solution comes from Bayesian predictive inference, which shows how to recover posterior parameter distributions from PPDs 
\citep{fong2023martingale, fortini2024quasi}.
One can either (i) sample from the posterior by recursively simulating future observations from a model's PPDs, or (ii) estimate the posterior variance by tracking how much the PPDs fluctuate as we condition on prefixes of the observed data, thereby constructing a normal approximation.
\citet{ng2026tabmgp} and \citet{fortini2026uncertainty} apply these respective techniques to calculate estimates of regression function posteriors for TabPFN in an offline setting.
However, the cost of computing the full sequence of PPDs required to construct either method's posterior estimates means that neither is practical for deployment in an online bandit setting, where a fresh posterior sample is needed for each arm at each round.

In this work, we break this computational bottleneck.
Building on \citet{fortini2026uncertainty}, we construct a normal approximation to the posterior, but using a more efficient consistent variance estimator.
More precisely, given an observed dataset of size $n$, we show that the asymptotic variance of the posterior can be consistently estimated using model snapshots along $O(\log n)$ geometrically spaced prefixes rather than all $n$ prefixes. 
We embed this posterior approximation for PFNs into a TS framework, which we call \emph{PFN-TS}.
PFN-TS further reduces per-round cost by reusing cached representations of the fitted PFN across rounds.
We derive a Bayesian regret bound for PFN-TS stated in terms of the maximum information gain for the Bayesian model defined by the broad prior and the error incurred by the PFN's approximation of its PPD.



While our methodology and theory apply to any PFN, we instantiate our method with TabICLv2 \citep{qu2026tabiclv2} (hereafter TabICL), a PFN for tabular regression and classification that
extends TabPFN \citep{hollmann2023tabpfn,hollmann2025tabpfn} with architectural
improvements and, critically, caching of the key-value attention matrices for a fixed
training set. In the bandit setting, where the same history is queried once per arm per
round, this reduces the per-round cost substantially. 
We evaluate it against a comprehensive
set of baselines on synthetic benchmarks, eight OpenML classification datasets, and the
Drink Less mobile health trial, following the same experimental protocol as \citet{deng2026bfts}. We find that PFN-TS substantially outperforms all baselines on almost all nonlinear tasks, while matching the performance of LinTS on linear tasks, and achieves the highest estimated policy value on the Drink Less trial.


\section{Posterior Uncertainty from Predictive Sequences}
\label{sec:uq}

\subsection{PFNs as amortized Bayesian inference}
\label{sec:tabicl}
Prior-Data Fitted Networks (PFNs)~\citep{muller2022transformers} are transformer models trained to approximate Bayesian PPDs in a single forward pass. 
To be precise, we define the following hierarchical Bayesian model:
\begin{equation} 
\label{eq:pfn_model}
\cD\cup\{(x,y)\} \mid \pi, N \overset{\text{i.i.d.}}{\sim} \pi, \quad \pi \sim \Pi, \quad N \sim \rho_N,
\end{equation}
where $\cD = \{(x_i,y_i)\}_{i=1}^N$ is a dataset of random size $N$, $\pi$ is a joint distribution over the pair $(x,y)$, $\Pi$ is a prior over such distributions, and $\rho_N$ is a distribution over dataset sizes. 
During training, synthetic datasets are drawn from this model, and the network minimizes negative log-likelihood given a training set and a query input. 
In other words, the network learns to approximate the mapping $(x, \cD) \mapsto p(y |x, \cD)$, where $p(y| x, \cD) = \int \pi(y|x) \Pi(\pi |\cD) \, d\Pi(\pi)$ is the PPD.


\subsection{Predictive sequences as posterior samplers}
\label{sec:martingale}

For a query point $x$ and a dataset $\cD = \{(x_i, y_i)\}_{i=1}^n$,
define the sequence of predictive means
$m_i(x) = m(x;\, \cD_{1:i})$, $i=0,\ldots,n$, where $m(x;\,\cD)$ is the mean
of the transformer's output distribution at query $x$ given context $\cD$, and
$\cD_{1:i}$ is the first $i$ observations in a fixed ordering of $\cD$.

\paragraph{Exact PPD predictive sequences.}
In the idealized setting when the PFN is an exact Bayesian PPD for prior $\Pi$, we have
$m_i(x) = \E_\Pi[y \mid x, \cD_{1:i}]$, the posterior mean of the response at $x$ given
the first $i$ observations. By the tower property,
$\E[m_{i+1}(x)\mid\cF_i]=m_i(x)$, so the sequence forms a Doob martingale with respect to
the filtration $\cF_i = \sigma(\cD_{1:i})$ and, under regularity conditions,
converges a.s.\ to a limit $m_\infty(x)$ by the Doob martingale convergence theorem.

To identify this limit, assume the data are i.i.d.\ draws from $\pi$ (i.e., model
\eqref{eq:pfn_model} is correctly specified). As $n \to \infty$, the posterior
$\Pi(\pi \mid \cD_{1:n})$ concentrates on the true $\pi$ a.s.\ (Bayesian consistency), so
$p(y \mid x, \cD_{1:n}) \to \pi(y \mid x)$ and hence
$m_\infty(x)=f_0(x):=\int y\,\pi(y\mid x)\,dy$ a.s.

A crucial consequence is that the law of $m_\infty(x)$ given $\cD_{1:n}$
coincides with the posterior over $f_0(x)$:
\begin{equation}
  \cL\bigl(m_\infty(x)|\cD_{1:n}\bigr) = p\bigl(f_0(x)|\cD_{1:n}\bigr).
  \label{eq:posterior_identification}
\end{equation}
This follows from the a.s.\ identity $m_\infty(x)=f_0(x)$. The randomness of
$m_\infty(x)$ given $\cD_{1:n}$ --- which comes from the unobserved future data that
would determine the limit --- is exactly the Bayesian uncertainty about
$f_0(x)$ \citep{fortini2024quasi}. Equivalently, the posterior uncertainty is the
conditional law of the martingale tail $m_\infty(x)-m_n(x)$ given $\cD_{1:n}$.
Martingality gives the centering and orthogonality structure, but a Gaussian
approximation requires second-order regularity of the predictive increments. Under
predictive-CLT conditions, this tail satisfies
\begin{equation}
  \cL\!\left(\sqrt{n}\,(m_\infty(x) - m_n(x)) \;\middle|\; \cD_{1:n}\right)
  \;\xrightarrow{w}\; \cN(0,\, V(x))
  \quad \mathbb{P}\text{-a.s.},
  \label{eq:pred_clt}
\end{equation}
as $n\to\infty$, giving a tractable Gaussian approximation to the posterior over the
latent mean $f_0(x)$ that can be used directly for Thompson sampling.

\paragraph{Approximate PPD predictive sequences.}
Since TabPFN and TabICL only provide an approximation to the exact PPD, the above properties do not hold automatically. 
\citet{fong2023martingale} proved that under martingale conditions on the predictive sequence, the limit $m_\infty(x)$ still exists and is identifiable as a random variable, although not necessarily equal to $f_0(x)$. Calling it a martingale posterior, they argue that it still offers a form of uncertainty quantification, which can be computed by simulating future values in the predictive sequence. 

Unfortunately, TabPFN and TabICL do not satisfy the exact martingale condition either.
To address this issue, \citet{ng2026tabmgp} and \citet{fortini2026uncertainty} introduce more relaxed sufficient conditions on the predictive sequences that still guarantee the existence of a limiting random variable, and provide empirical evidence that TabPFN's predictive sequences exhibit such limiting behavior, thereby enabling uncertainty quantification in an offline setting.
While they worked with predictive CDFs rather than means, a similar analysis applies to the predictive mean sequence as well.
We now state a similar set of assumptions for our setting, that we will use throughout the rest of the paper.




\begin{assumption}[Asymptotic variance]
\label{assump:covariance}
There $\exists~V(x) > 0$ s.t. \scalebox{0.9}{$i^2 \E\left[(m_i(x) - m_{i-1}(x))^2 |\cF_{i-1}\right] \to V(x)$} a.s.\ as $i \to \infty$.
\end{assumption}

\begin{assumption}[Quasi-martingale condition]
\label{assump:quasi_martingale}
\scalebox{0.9}{$\displaystyle\sum_{i \geq 1} \sqrt{i}\,\Bigl(\E\!\Bigl[\bigl(\E\!\left[m_i(x) - m_{i-1}(x) \mid \cF_{i-1}\right]\bigr)^{\!2}\Bigr]\Bigr)^{\!1/2} < \infty.$}
\end{assumption}

\begin{assumption}[Bounded moments]
\label{assump:moments}
There exist $C > 0$ and $\epsilon > 0$ such that
$\E[(m_i(x) - m_{i-1}(x))^{4+\epsilon}] \leq C / i^{4+\epsilon}$.
\end{assumption}

Following \citet[Theorem~4.3]{fortini2026uncertainty} and under these assumptions, one can establish a predictive CLT similar to \eqref{eq:pred_clt} whose asymptotic variance $V(x)$ can be estimated by $\frac{1}{n}\sum_{i=1}^n i^2\,(m_i(x) - m_{i-1}(x))^2$, but evaluating this requires a forward pass at all $n$ dataset prefixes, which is prohibitive in an online bandit setting where a fresh variance estimate is needed for each arm at each round.


\subsection{Subsampled CLT for predictive means}
\label{sec:subsampled_clt}

Rather than evaluating all $n$ increments $m_i(x) - m_{i-1}(x)$, we compute block increments at
$O(\log n)$ geometrically spaced prefix indices $2 = t_0 < t_1 < \cdots < t_J \le n$,
where $t_J$ is the last point in the geometric sequence not exceeding $n$.
For interval $(t_{j-1},t_j]$, set
$D_j=m_{t_j}(x)-m_{t_{j-1}}(x)$ and
$w_j=t_jt_{j-1}/(t_j-t_{j-1})$.

To see why this weighting is natural, write each increment as a martingale-difference
term plus a predictable drift. When the drift is negligible, martingale orthogonality
suggests $\E[D_j^2]\approx V(x)\sum_{i=t_{j-1}+1}^{t_j}i^{-2}
\approx V(x)(1/t_{j-1}-1/t_j)$, so the harmonic weights $w_j$ achieve the same
identification target as the $i^2$ weights in the estimator of
\citet{fortini2026uncertainty} on a coarser grid.

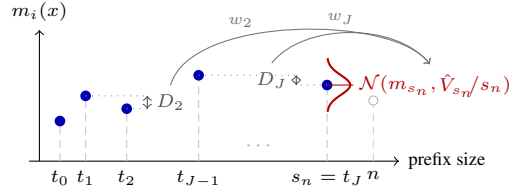
\begin{wrapfigure}{r}{0.5\textwidth}
\centering
\begin{tikzpicture}[x=0.68cm,y=0.68cm, font=\scriptsize]
\draw[->] (0.25,0) -- (7.45,0) node[right] {prefix size};
\draw[->] (0.45,0) -- (0.45,2.55) node[above] {$m_i(x)$};

\coordinate (p0) at (0.85,0.78);
\coordinate (p1) at (1.35,1.27);
\coordinate (p2) at (2.15,1.02);
\coordinate (p3) at (3.55,1.67);
\coordinate (p4) at (6.05,1.48);

\draw[gray!45, densely dashed] (0.85,0) -- (p0);
\draw[gray!45, densely dashed] (1.35,0) -- (p1);
\draw[gray!45, densely dashed] (2.15,0) -- (p2);
\draw[gray!45, densely dashed] (3.55,0) -- (p3);
\draw[gray!45, densely dashed] (6.05,0) -- (p4);
\draw[gray!35, densely dashed] (6.95,0) -- (6.95,1.18);

\foreach \p in {p0,p1,p2,p3,p4} {
  \fill[blue!70!black] (\p) circle (2pt);
}
\fill[white, draw=gray!60] (6.95,1.18) circle (1.7pt);

\node[below] at (0.85,0) {$t_0$};
\node[below] at (1.35,0) {$t_1$};
\node[below] at (2.15,0) {$t_2$};
\node[below] at (3.55,0) {$t_{J-1}$};
\node[below] at (6.05,0) {$s_n=t_J$};
\node[below] at (6.95,0) {$n$};
\node[gray!70] at (4.75,0.3) {$\cdots$};

\draw[<->, black!70] (2.55,1.27) -- (2.55,1.02)
  node[midway,right] (Dtwo) {$D_2$};
\draw[black!45, dotted] (p1) -- (2.55,1.27);
\draw[black!45, dotted] (p2) -- (2.55,1.02);
\draw[<->, black!70] (5.45,1.67) -- (5.45,1.48)
  node[midway,left] (DJ) {$D_J$};
\draw[black!45, dotted] (p3) -- (5.45,1.67);
\draw[black!45, dotted] (p4) -- (5.45,1.48);

\draw[red!70!black, thick, domain=-1.45:1.45, samples=60, variable=\u]
  plot ({6.05 + 0.45*exp(-\u*\u/0.55)}, {1.48 + 0.36*\u});
\draw[red!70!black] (6.05,1.48) -- (6.56,1.48);
\node[red!70!black, right] at (6.50,1.48) {$\cN(m_{s_n},\, $};
\node[red!70!black, right] (vhatlabel) at (8.05,1.48) {$\hat V_{s_n}$};
\node[red!70!black, right] at (8.63,1.48) {$/s_n)$};
\draw[->, black!55] (Dtwo.north)
  .. controls (3.4,2.85) and (7.35,2.85) .. (vhatlabel.north west)
  node[pos=0.34, above, font=\tiny] {$w_2$};
\draw[->, black!55] (DJ.north)
  .. controls (5.6,2.7) and (7.40,2.7) .. (vhatlabel.north west)
  node[pos=0.45, above, font=\tiny] {$w_J$};
\end{tikzpicture}
\caption{SubCLT evaluates the predictive mean on a geometric prefix grid. The
trajectory increments $D_j$ estimate the CLT variance used for the Thompson sample at
the latest refresh point $s_n$.}
\label{fig:subclt}
\vspace{-15pt}
\end{wrapfigure}
The resulting variance estimator is
\begin{equation}
  \hat{V}_{s_n}(x) = \frac{1}{J} \sum_{j=1}^J w_j D_j^2,
  \label{eq:variance_estimator}
\end{equation}
where $s_n=t_J$ is the latest refresh point. The cached sampler uses the snapshot
Gaussian approximation
$\cN(m_{s_n}(x),\hat V_{s_n}(x)/s_n)$. With $J=O(\log_b n)$ terms, the default
$b=2$ requires $O(\log_2 n)$ forward passes per arm per round. Pseudocode is in
\Cref{alg:subsampled_clt} (see also \Cref{fig:subclt}); the formal guarantee is as follows.

\begin{theorem}[Subsampled predictive CLT]
\label{thm:clt}
Under \Cref{assump:covariance,assump:quasi_martingale,assump:moments},
\[
  \cL\!\left(\sqrt{q}\,\bigl(m_\infty(x) - m_q(x)\bigr) \;\middle|\; \cD_{1:q}\right)
  \xrightarrow{w} \cN(0, V(x))
  \quad \mathbb{P}\text{-a.s.}
\]
as $q\to\infty$. Therefore the same CLT holds along any diverging sequence of refresh
points. Moreover, for a history of size $n$ and the geometric grid obtained by iterating
$t_{j+1} = \max(t_j{+}1,\,\lfloor b\,t_j\rfloor)$ up to the last point $t_J \le n$,
with fixed $b > 1$, the estimator $\hat{V}_{s_n}(x)$ defined in \Cref{eq:variance_estimator}
satisfies $\hat{V}_{s_n}(x) \xrightarrow{p} V(x)$, where $s_n=t_J$.
\end{theorem}



\section{Thompson Sampling with PFNs}
\label{sec:method}

\subsection{Problem setup}
\label{sec:setup}
We consider the standard contextual bandit over $T$ rounds. 
At each round
$t \in [T]$, the learner observes a context $X_t \in \cX \subseteq [0,1]^p$, selects an
action $A_t \in \cA$ where $|\cA| = K$, and receives a reward
\begin{equation}
  R_t = f_0(X_t, A_t) + \varepsilon_t,
\end{equation}
where $f_0 : \cX \times \cA \to \R$ is the unknown mean reward function and $\varepsilon_t$
is independent mean-zero noise with $|\varepsilon_t| \leq B$ almost surely. The history at
round $t$ is $\cH_t = \{(X_s, A_s, R_s)\}_{s=1}^t$. The goal is to minimize cumulative
regret against the oracle policy $A^*(x) \in \argmax_{a \in \cA} f_0(x, a)$:
\[
  \Regret_T = \sum_{t=1}^T \bigl(f_0(X_t, A^*(X_t)) - f_0(X_t, A_t)\bigr).
\]
Thompson Sampling maintains a posterior $\Pi(\cdot \mid \cH_{t-1})$ over reward functions
and at each round draws $\hat{f}_t \sim \Pi(\cdot \mid \cH_{t-1})$ and plays
$A_t = \argmax_a \hat{f}_t(X_t, a)$.

\subsection{Context parameterization}
\label{sec:parameterization}

A design choice in applying any tabular model to contextual bandits is how to represent
arm identity in the input. We consider two strategies.

\textbf{Disjoint encoding} maintains $K$ independent models, one per arm. For arm $k$,
the history is $\cD^{(k)} = \{(X_s, R_s) : s \leq t,\, A_s = k\}$, and the predictive mean
at query $X_t$ is $m(X_t;\, \cD^{(k)})$. Each model sees only the data from its own arm.

\textbf{Joint encoding (one-hot)} maintains a single model with arm identity appended as
a one-hot vector: the feature for arm $k$ and context $x$ is $\tilde{x}_k = (x, e_k) \in
\R^{p+K}$, where $e_k \in \{0,1\}^K$. The joint history $\cD^t = \{(\tilde{X}_{s,A_s}, R_s)
: s \leq t\}$ is used for all arms. This allows observations from one arm to inform predictions for another.

Neither encoding universally dominates.
Joint encoding is preferable when arms share structure, since it pools observations; disjoint encoding is more effective when arm reward functions are unrelated. 
A third option, \emph{block encoding}, expands the feature space to $\R^{pK}$ by placing $x$ in the
$k$-th block and zero-padding elsewhere, but the $K$-fold increase in dimensionality makes it less practical for PFNs.


\subsection{The PFN-TS algorithm}
\label{sec:adaptive_encoding}

After a round-robin warm-up of $\tau$ observations per arm, PFN-TS enters its main
Thompson Sampling loop. We describe how it operates under each encoding, then explain
the adaptive selection rule.

\textbf{Disjoint encoding}.
Each arm $k$ maintains its own model and history $\cD^{(k)}$ of size $n_k = |\cD^{(k)}|$. At
each round $t$, \textproc{SubCLT} is called with $\cD^{(k)}$ and query $X_t$, evaluating
the predictive mean at a geometrically spaced grid $t_0 < t_1 < \cdots < t_J \leq n_k$
of prefix sizes. These grid points are the \emph{refresh times} for arm $k$: at each
$t_j$, the KV cache for prefix $\cD^{(k)}_{1:t_j}$ is stored for reuse in subsequent
rounds. Let $s_{t,k}=t_J$ be the latest refresh point for arm $k$. The Thompson sample is
drawn from the snapshot Gaussian approximation
\[
  \tilde{r}_{t,k} \sim
  \cN\!\left(
    m(X_t;\cD^{(k)}_{1:s_{t,k}}),\,
    \hat{V}_{s_{t,k}}(X_t) / s_{t,k}
  \right).
\]
The arm with the highest sample is selected: $A_t = \argmax_k \tilde{r}_{t,k}$.

\textbf{Joint encoding}.
Under joint encoding the setup is similar, but with a single shared model and history
$\cD$ of size $n = |\cD|$, queried at the augmented context $(X_t, e_k)$ for each arm
$k$. The refresh times $t_0 < t_1 < \cdots < t_J \leq n$ are now based on the \emph{total}
history and shared across arms: at each $t_j$, the KV cache for prefix $\cD_{1:t_j}$ is
built once and reused for all $K$ arm queries, reducing the per-round SubCLT cost from
$O(K \log n_k)$ to $O(\log n)$. Since the variance estimate is computed at the latest
shared refresh point $s_t=t_J$, the Thompson sample uses the shared snapshot size:
\[
  \tilde{r}_{t,k} \sim
  \cN\!\left(
    m((X_t,e_k);\cD_{1:s_t}),\,
    \hat{V}_{s_t}(X_t,e_k) / s_t
  \right).\footnote{For joint encoding, a fully Bayesian Thompson sampler would sample the arm-reward vector jointly from the predictive law induced by the shared context. Our implementation samples independently from marginal Gaussian approximations, so the joint-encoding variant should be viewed as a heuristic empirical extension of the disjoint sampler.}
\]

\textbf{Adaptive encoding selection}.
Since neither encoding dominates in general, PFN-TS runs both in parallel during a
\emph{dual-encoding} phase and selects between them via the Continuous Ranked Probability
Score (CRPS) \citep{gneiting2007strictly}: for a predictive CDF $\hat{F}$ and outcome
$r$,
\begin{equation}
  \CRPS(\hat{F}, r) = \int_{-\infty}^\infty \bigl(\hat{F}(y) - \mathbf{1}\{y \geq r\}\bigr)^2 \mathrm{d}y.
  \label{eq:crps}
\end{equation}
CRPS is a strictly proper scoring rule, so the better-calibrated encoding accumulates
lower cumulative CRPS. At pre-specified \emph{switch times} $S = \{s_1 < \cdots <
s_L\}$, cumulative CRPS is compared and the better encoding becomes active; after $s_L$,
the challenger is discarded. Snapshot caching at refresh times means CRPS evaluation
requires no additional forward passes: observations since the last switch are replayed
against the stored snapshots. Full pseudocode is in
Algorithm~\ref{alg:tabicl_ts} (\Cref{app:algorithm}).

\textbf{Related work}.
\citet{zhang2026generative} propose a closely related framework in which TS uncertainty is recast as arising from missing counterfactual outcomes: a generative sequence model, pretrained on bandit data from previous tasks, imputes all unobserved arm outcomes at each round, and a policy is then fit on the
imputed complete dataset. 
Relative to their approach, ours differs in three ways: (i) we use a generic PFN without any
task-specific training; (ii) we extract uncertainty from the predictive distribution
directly via SubCLT rather than by generating counterfactual outcomes for all arms,
avoiding the cost of full imputation passes; and (iii) we operate in the single-task
setting without requiring offline data from previous tasks in the same distribution.

\section{Bayesian Regret of PFN-TS}
\label{sec:regret_analysis}

To bound the Bayesian regret of PFN-TS, we compare it to exact Thompson Sampling in the
Bayesian model associated with the prior $\Pi$. The regret analysis uses the
disjoint-arm Bayesian design induced by \eqref{eq:pfn_model}: for each arm
$a$, an independent joint law $\pi_a\sim\Pi_a$ is drawn, with common context
marginal on $\cX$, and rewards queried at arm $a$ are generated from the conditional law
$\pi_a(\cdot\mid x)$. Thus the prior over mean reward functions is
$\Pi=\bigotimes_{a=1}^K\Pi_a$, where
$f_0(x,a)=\int y\,\pi_a(dy\mid x)$. Contexts are exogenous draws from the common
context marginal. Conditional on the realized contexts and selected actions, the
posterior update for arm $a$ is the arm-specific posterior predictive update from
\eqref{eq:pfn_model} restricted to the observed pairs in $\cD^{(a)}$; adaptivity enters
only through which contexts are queried for each arm.
Under this Bayesian design, regret is
\[
  \E[\Regret_T] = \E\!\left[\sum_{t=1}^T \bigl(f_0(X_t, A_t^*) - f_0(X_t, A_t)\bigr)\right],
\]
with expectation over the prior, contexts, rewards, and algorithmic randomness. The result
is stated for disjoint encoding. Under joint encoding, the transformed contexts need not
arise from arm-specific copies of \eqref{eq:pfn_model}, so the same argument does not
directly apply.

Let $P_t^\Pi(\cdot\mid x,\cD^{(1:K)})$ be the action distribution induced by exact
posterior sampling under $\Pi$ after context $x$ and arm histories
$\cD^{(1:K)}=(\cD^{(1)},\ldots,\cD^{(K)})$. Let
$Q_t(\cdot\mid x,\cD^{(1:K)})$ be the corresponding action distribution used by PFN-TS.

\begin{theorem}[Bayesian regret with sampler approximation]
\label{thm:main}
Suppose the posterior expected reward range is uniformly bounded by a finite constant
$2B_R$: for every finite arm-specific history collection $\cD^{(1:K)}$,
\[
  \sup_{x\in\cX,\;a,a'\in\cA}
  \E_\Pi\!\left[
    |f_0(x,a)-f_0(x,a')|
    \mid \cD^{(1:K)}
  \right]
  \leq 2B_R.
\]
Then
\[
  \E[\Regret_T^{\mathrm{PFN}}]
  \leq
  \E[\Regret_T^{\Pi\text{-}\mathrm{TS}}]
  + 2B_R\sum_{t=1}^T (T-t+1)\varepsilon_t.
\]
Here $\Regret_T^{\Pi\text{-}\mathrm{TS}}$ is the regret of exact Thompson Sampling under
$\Pi$ and
\[
  \varepsilon_t
  =
  \sup_{x,\cD^{(1:K)}}
  \TV\!\left(
    Q_t(\cdot\mid x,\cD^{(1:K)}),
    P_t^\Pi(\cdot\mid x,\cD^{(1:K)})
  \right).
\]
\end{theorem}

The theorem is a coupling statement. If PFN-TS and exact $\Pi$-TS have the same history at
round $t$, a maximal coupling makes their actions differ with probability at most
$\varepsilon_t$. Once their actions differ, their future histories may differ, which
produces the factor $T-t+1$. The first term is the regret of the ideal Bayesian algorithm
for the prior approximated by TabICL; the second term measures the cost of using the
PFN-TS sampler instead of exact posterior sampling.

\begin{remark}[Bounding the exact Bayesian term]
The regret of exact $\Pi$-TS can be bounded using the standard information-ratio
framework. Under disjoint encoding and a product prior, the factorization argument of
\citet{deng2026bfts} gives
\[
  \E[\Regret_T^{\Pi\text{-}\mathrm{TS}}]
  \leq
  \sqrt{2T\,\bar{\Gamma}_T(\Pi)\,\gamma_T(\Pi)}.
\]
Here $\bar{\Gamma}_T(\Pi)$ is the worst-case information ratio of exact $\Pi$-TS, and
$\gamma_T(\Pi)$ is the corresponding disjoint maximum information gain. If posterior
reward distributions under $\Pi$ are uniformly sub-Gaussian with proxy $\sigma_R^2$, then
finite-action Thompson Sampling bounds suggest $\bar{\Gamma}_T(\Pi)\lesssim K\sigma_R^2$
\citep{gouverneur2023thompson}. The term $\gamma_T(\Pi)$ depends on the complexity of the
prior $\Pi$; \Cref{app:proof_main} gives the formal definition and comparison rates for
linear and H\"older classes.
\end{remark}

\begin{remark}[Bounding the sampler error]
\Cref{prop:sampler_error} bounds $\varepsilon_t$ by a Gaussian approximation error for exact
posterior sampling under $\Pi$ plus a KL-type discrepancy between the Gaussian from exact-PPD SubCLT and that constructed from applying SubCLT to the PFN's PPD approximations. The latter term depends on predictive mean error
and the induced variance ratio, so the coupling error is controlled by approximation
properties of the PFN.
\end{remark}

\section{Numerical Experiments}
\label{sec:experiments}

Across synthetic and real-data benchmarks, PFN-TS improves over standard baselines when
reward functions are nonlinear or heterogeneous, while remaining competitive on problems
where simpler inductive biases are well matched to the data. \Cref{fig:main_regret} shows
selected regret curves; full trajectories are in \Cref{app:full_synthetic,app:full_openml}.

\begin{figure}[H]
  \centering
  \begin{subfigure}[t]{0.24\textwidth}
    \includegraphics[width=\textwidth]{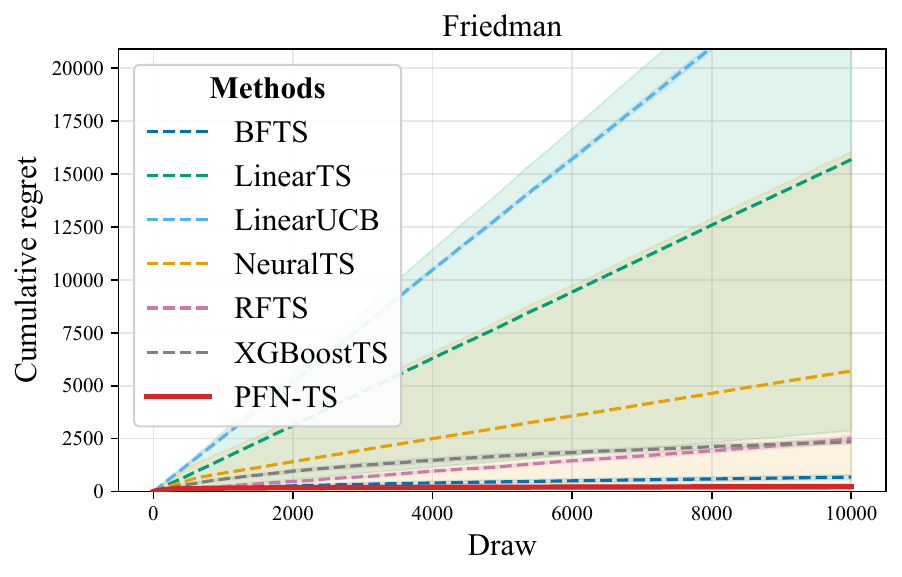}
    \caption{Friedman}
  \end{subfigure}\hfill
  \begin{subfigure}[t]{0.24\textwidth}
    \includegraphics[width=\textwidth]{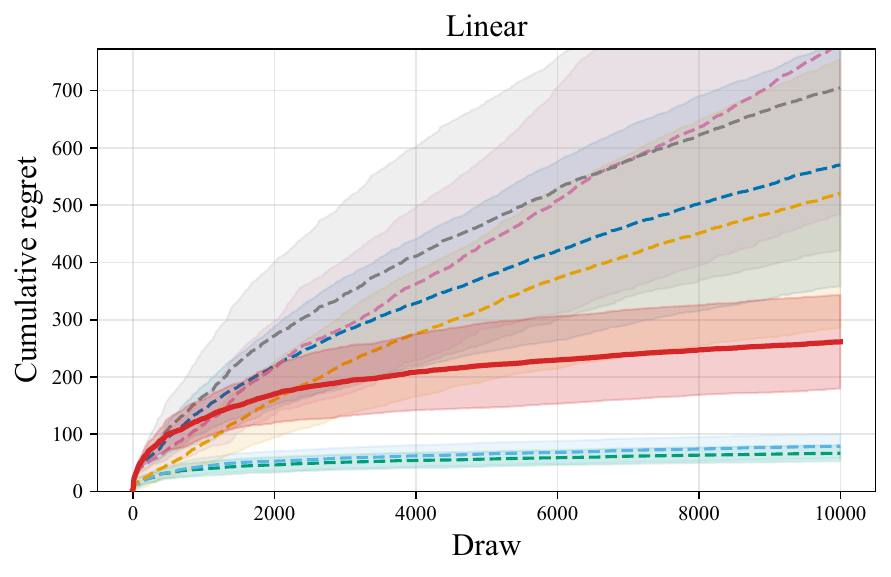}
    \caption{Linear}
  \end{subfigure}\hfill
  \begin{subfigure}[t]{0.24\textwidth}
    \includegraphics[width=\textwidth]{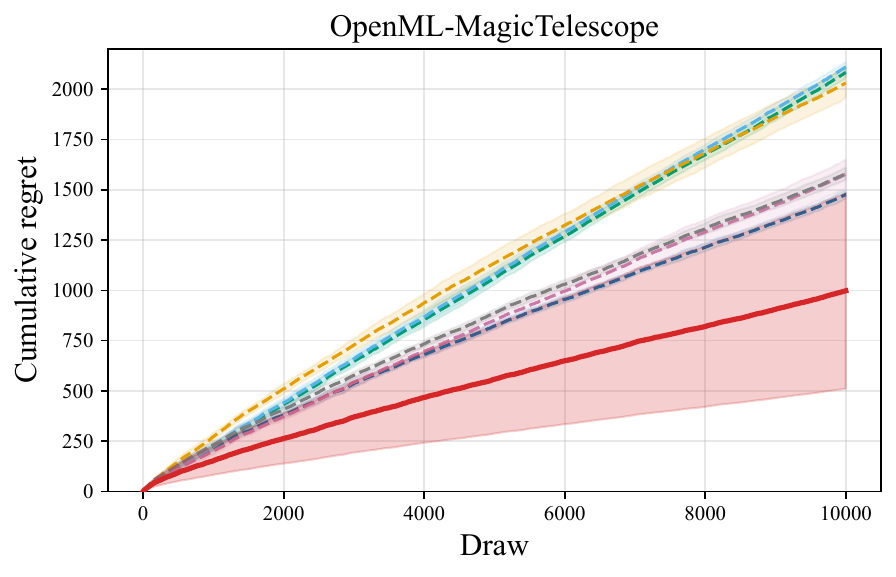}
    \caption{MagicTelescope}
  \end{subfigure}\hfill
  \begin{subfigure}[t]{0.24\textwidth}
    \includegraphics[width=\textwidth]{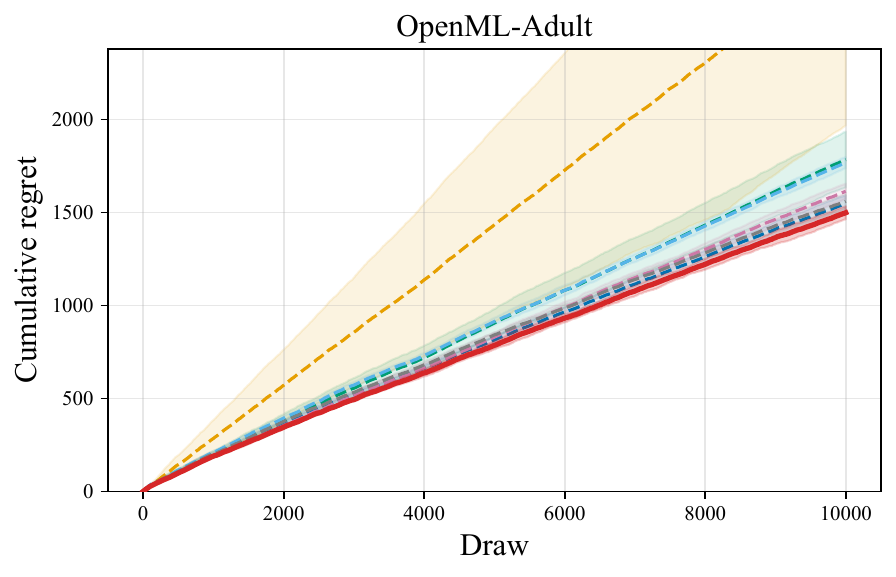}
    \caption{Adult}
  \end{subfigure}
  \caption{Selected cumulative regret trajectories (mean $\pm$ SD, $R = 5$ replications,
  $T = 10{,}000$). Left to right: Friedman (synthetic), Linear (synthetic),
  MagicTelescope (OpenML), Adult (OpenML). Full results are in
  \Cref{app:full_synthetic,app:full_openml}.}
  \label{fig:main_regret}
\end{figure}

\subsection{Experimental Setup}
\label{sec:setup_exp}

We compare PFN-TS with \textbf{LinTS} (linear Thompson Sampling), \textbf{LinUCB}
\citep{li2010contextual}, \textbf{NeuralTS} \citep{zhang2021neural}, \textbf{BFTS}
\citep{deng2026bfts}, \textbf{RFTS}, and \textbf{XGBoostTS} \citep{nilsson2024tree}. The
baseline implementations and fixed settings follow \citet{deng2026bfts}, with the
method-specific choices listed in \Cref{app:hyperparams}. Synthetic and OpenML
experiments use horizon $T = 10{,}000$ and $R = 5$ independent replications. PFN-TS uses
geometric base $b = 2$, no context-window truncation, and a warm-up of $\tau = 5$
round-robin pulls per arm; BFTS uses the same warm-up convention. LinTS and NeuralTS are
tuned on a separate pilot run, while the other baselines use fixed settings.

\subsection{Synthetic Benchmarks}
\label{sec:synthetic}

We evaluate on three families of data-generating processes (DGPs): (i)~\textbf{Linear}
rewards with Gaussian noise, (ii)~\textbf{Friedman} nonlinear rewards with varying
sparsity, correlation structure, and heteroscedastic noise, and (iii)~\textbf{SynBART},
where each arm's reward function is sampled from a BART prior, providing a
correctly-specified benchmark for BFTS. Full DGP definitions are in \Cref{app:dgp}.

\Cref{tab:synthetic_final_regret} reports final cumulative regret across all scenarios.
PFN-TS has the best average rank overall (1.38) and achieves the lowest regret on all six
Friedman variants, often by a large margin over tree and neural baselines. The main
exceptions align with model specification: LinTS is best on the Linear benchmark, while
BFTS narrowly edges PFN-TS on SynBART, where the reward functions are drawn from the BART
prior used by BFTS. The selected trajectories in \Cref{fig:main_regret} illustrate this
pattern: PFN-TS separates quickly on nonlinear Friedman rewards, while LinTS remains the
right inductive bias on the linear task.

\begin{table}[t]
  \caption{Final cumulative regret on synthetic benchmarks (mean $\pm$ SE, $T = 10{,}000$,
  $R = 5$ replications). \textbf{Bold}: lowest mean per scenario; avg rank computed across
  all scenarios.}
  \label{tab:synthetic_final_regret}
  \centering
\resizebox{\textwidth}{!}{
\begin{tabular}{lccccccccc}
\toprule
Method & Friedman & Friedman-Hetero. & Friedman-Sparse & Friedman-Sparse-Dis. & Friedman2 & Friedman3 & Linear & SynBART & Rank \\ \midrule
\textbf{PFN-TS (ours)} & \textbf{233.0} $\pm$ \textbf{94.1} & \textbf{243.2} $\pm$ \textbf{67.4} & \textbf{306.3} $\pm$ \textbf{78.5} & \textbf{568.6} $\pm$ \textbf{57.5} & \textbf{90.2} $\pm$ \textbf{21.3} & \textbf{136.4} $\pm$ \textbf{14.6} & 261.4 $\pm$ 81.5 & 51.2 $\pm$ 31.0 & 1.375\\
BFTS & 680.3 $\pm$ 98.5 & 777.1 $\pm$ 126.6 & 835.5 $\pm$ 91.6 & 684.4 $\pm$ 66.1 & 367.1 $\pm$ 74.3 & 559.3 $\pm$ 67.8 & 570.4 $\pm$ 212.5 & \textbf{50.9} $\pm$ \textbf{20.0} & 2.375 \\ 
LinearTS & 15690.6 $\pm$ 12812.3 & 26214.7 $\pm$ 149.0 & 10427.3 $\pm$ 12770.9 & 6529.8 $\pm$ 1210.7 & 524.8 $\pm$ 168.8 & 5260.7 $\pm$ 10521.5 & \textbf{66.9} $\pm$ \textbf{14.1} & 301.1 $\pm$ 76.0 & 5.125 \\ 
LinearUCB & 26222.5 $\pm$ 215.8 & 26214.7 $\pm$ 149.0 & 23106.6 $\pm$ 6305.4 & 7743.7 $\pm$ 2746.3 & 682.1 $\pm$ 120.2 & 26304.4 $\pm$ 244.2 & 79.2 $\pm$ 21.2 & 274.0 $\pm$ 113.1 & 6.000 \\ 
NeuralTS & 5697.3 $\pm$ 10349.9 & 20311.5 $\pm$ 10242.4 & 9673.3 $\pm$ 9809.5 & 1376.3 $\pm$ 151.2 & 148.2 $\pm$ 114.0 & 728.6 $\pm$ 905.3 & 520.7 $\pm$ 234.5 & 198.2 $\pm$ 73.4 & 4.000 \\ 
RFTS & 2504.5 $\pm$ 139.0 & 2562.1 $\pm$ 218.8 & 2688.8 $\pm$ 259.4 & 7579.2 $\pm$ 218.9 & 1058.4 $\pm$ 85.2 & 909.6 $\pm$ 123.5 & 781.7 $\pm$ 297.2 & 106.7 $\pm$ 47.7 & 4.875 \\ 
XGBoostTS & 2354.4 $\pm$ 115.4 & 2432.1 $\pm$ 119.8 & 2694.1 $\pm$ 82.4 & 4962.3 $\pm$ 133.1 & 780.9 $\pm$ 33.5 & 1738.6 $\pm$ 49.3 & 704.9 $\pm$ 283.6 & 66.7 $\pm$ 28.3 & 4.250 \\
\bottomrule
\end{tabular}}
\end{table}

\textbf{SubCLT uncertainty diagnostics.}
To isolate uncertainty quality from adaptive data collection, we also evaluate SubCLT in
offline regression experiments. For Linear and Friedman DGPs, we fit on
$n \in \{16,64,256,1024\}$ observations, construct nominal 95\% intervals for the latent
mean at held-out queries, and plot empirical coverage against mean interval length. BART
serves as an MCMC posterior baseline. \Cref{fig:clt_frontier} shows a mixed but useful
diagnostic: TabICL-SubCLT undercovers on the Linear DGP, but on the nonlinear Friedman DGP
it attains near-nominal coverage with substantially shorter intervals than BART. Details
are in \Cref{app:calibration}.

\begin{figure}[t]
  \vspace{-1.0em}
  \centering
  \begin{subfigure}[t]{0.40\textwidth}
    \includegraphics[width=\textwidth]{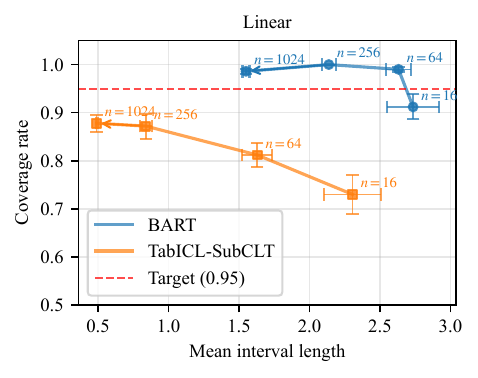}
  \end{subfigure}\hspace{0.04\textwidth}
  \begin{subfigure}[t]{0.40\textwidth}
    \includegraphics[width=\textwidth]{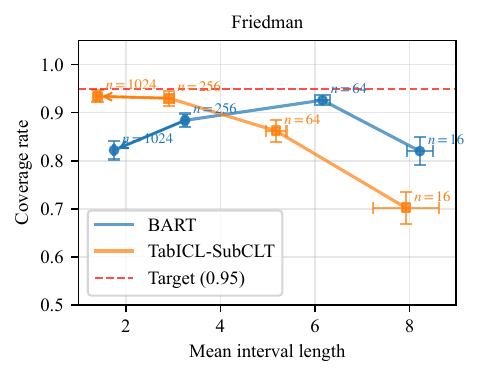}
  \end{subfigure}
  \caption{Coverage--length diagnostics for nominal 95\% intervals. Each marker is one
  sample size; lower interval length at fixed coverage is better.}
  \label{fig:clt_frontier}
  \vspace{-1.25em}
\end{figure}

\subsection{OpenML Benchmarks}
\label{sec:openml}

We transform eight OpenML \citep{vanschoren_openml_2014} classification tasks into contextual
bandit problems following standard protocol \citep{riquelme2018deep,zhang2021neural,nilsson2024tree,deng2026bfts}:
each class label becomes an arm, and the reward is 1 for choosing the correct class and 0
otherwise. We use Adult, Covertype, EEGEyeState, GasDrift, MagicTelescope, Mushroom,
PageBlocks, and Shuttle. 

\Cref{tab:openml_final_regret} reports final cumulative regret and average rank across all
eight datasets. PFN-TS again has the best average rank (1.88) and obtains the lowest final
regret on six of the eight tasks: Adult, Covertype, EEGEyeState, GasDrift,
MagicTelescope, and Mushroom. BFTS performs best on PageBlocks and Shuttle.
Some representative OpenML trajectories are shown in \Cref{fig:main_regret}, with full
per-dataset results in \Cref{app:full_openml}.
We note that the only OpenML dataset where PFN-TS is far from the best is Shuttle, which has a large number of arms (7) and a highly imbalanced reward structure.
PFN-TS is less stable on this dataset, which we believe can be mitigated by further tuning of the encoding strategy; we leave this to future work.

\begin{table}[t]
  \caption{Final cumulative regret on OpenML benchmarks (mean $\pm$ SE). The rightmost
  column reports average rank across all eight datasets; lower is better.}
  \label{tab:openml_final_regret}
  \centering
\resizebox{\textwidth}{!}{
\begin{tabular}{lccccccccc}
\toprule
Method & Adult & Covertype & EEGEyeState & GasDrift & MagicTelescope & Mushroom & PageBlocks & Shuttle & Rank \\
\midrule
\textbf{PFN-TS (ours)} & \textbf{1499.2} $\pm$ \textbf{35.0} & \textbf{3441.0} $\pm$ \textbf{457.7} & \textbf{625.6} $\pm$ \textbf{12.7} & \textbf{350.8} $\pm$ \textbf{32.3} & \textbf{998.0} $\pm$ \textbf{486.6} & \textbf{37.0} $\pm$ \textbf{4.8} & 295.6 $\pm$ 16.2 & 681.4 $\pm$ 572.3 & 1.875 \\
BFTS & 1548.8 $\pm$ 44.1 & 3642.0 $\pm$ 36.2 & 2270.6 $\pm$ 37.6 & 878.8 $\pm$ 59.6 & 1476.6 $\pm$ 20.1 & 56.0 $\pm$ 9.0 & \textbf{272.0} $\pm$ \textbf{17.1} & \textbf{108.2} $\pm$ \textbf{4.9} & 2.375 \\
LinearTS & 1785.0 $\pm$ 150.4 & 3457.8 $\pm$ 99.2 & 3649.2 $\pm$ 50.8 & 489.2 $\pm$ 11.1 & 2083.0 $\pm$ 27.2 & 493.8 $\pm$ 93.1 & 351.4 $\pm$ 14.9 & 639.6 $\pm$ 19.6 & 4.750 \\
LinearUCB & 1768.2 $\pm$ 30.8 & 3705.2 $\pm$ 52.2 & 3622.8 $\pm$ 10.3 & 930.8 $\pm$ 23.2 & 2110.0 $\pm$ 23.6 & 401.4 $\pm$ 3.4 & 370.4 $\pm$ 16.5 & 767.8 $\pm$ 43.1 & 5.625 \\
NeuralTS & 2984.8 $\pm$ 1017.6 & 5014.8 $\pm$ 1268.8 & 4977.6 $\pm$ 38.7 & 7662.6 $\pm$ 182.1 & 2031.0 $\pm$ 74.5 & 1711.8 $\pm$ 1924.3 & 412.4 $\pm$ 24.2 & 429.0 $\pm$ 332.5 & 6.375 \\
RFTS & 1614.4 $\pm$ 41.7 & 3601.8 $\pm$ 57.3 & 2666.0 $\pm$ 66.7 & 2119.4 $\pm$ 356.5 & 1577.8 $\pm$ 71.0 & 69.6 $\pm$ 26.6 & 292.6 $\pm$ 19.4 & 176.6 $\pm$ 42.6 & 3.500 \\
XGBoostTS & 1560.2 $\pm$ 40.1 & 3535.2 $\pm$ 82.8 & 2113.2 $\pm$ 36.1 & 1215.4 $\pm$ 32.2 & 1579.2 $\pm$ 29.7 & 81.4 $\pm$ 6.4 & 306.2 $\pm$ 12.4 & 182.6 $\pm$ 93.2 & 3.500 \\
\bottomrule
\end{tabular}}

\end{table}

\subsection{Off-policy Evaluation: Drink Less mHealth Trial}
\label{sec:ope}

\begin{wrapfigure}{r}{0.46\textwidth}
  \vspace{-1.25em}
  \centering
  \includegraphics[width=0.45\textwidth]{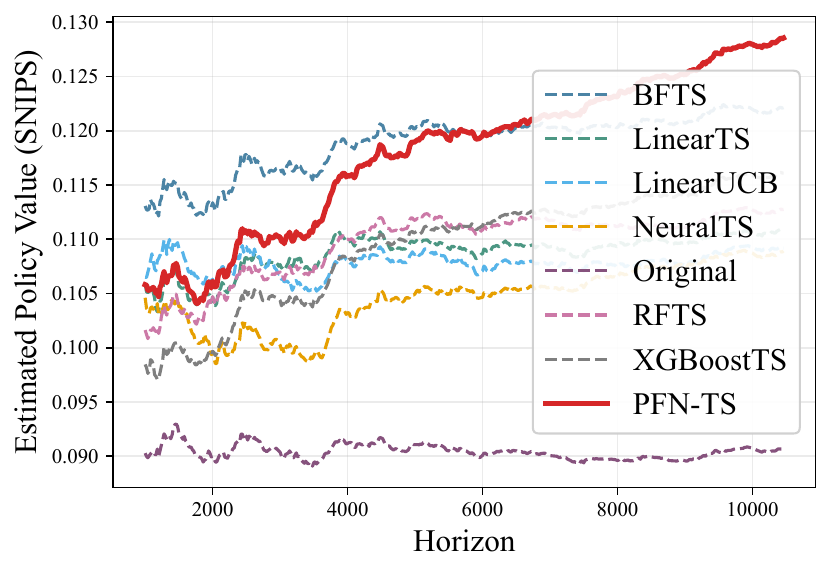}
  \caption{Estimated policy value (SNIPS) on the Drink Less trial. Higher is better.}
  \label{fig:ope}
  \vspace{-1em}
\end{wrapfigure}

To assess performance on real-world logged data, we apply PFN-TS to the Drink Less
micro-randomized trial \citep{bell2020drinkless}, a 30-day study in which $n = 349$
participants at risk of hazardous drinking were randomized daily among three push-notification
actions (no message, standard message, tailored message) with static propensities
$(0.4, 0.3, 0.3)$. The reward is a binary proximal engagement indicator (app opening within
the following hour). Following \citet{deng2026bfts}, we unfold the panel into
$T = n \times 30 = 10{,}470$ sequential decision points, preserving within-participant
temporal order, and simulate each algorithm's policy in offline replay. Policy value is
estimated using self-normalized importance sampling
(SNIPS;~\citet{swaminathan2015self}), with a doubly-robust (DR) estimator as a robustness
check \citep{dudik2011doubly}. \Cref{fig:ope} shows that PFN-TS has the highest estimated
SNIPS value by the final horizon, but performs worse than BFTS on shorter horizons. 
OPE diagnostics (importance weight distributions and DR
estimates) are in \Cref{app:ope_diagnostics}.

\subsection{Ablation Studies}
\label{sec:ablation}

We examine three design choices of PFN-TS; full ablation plots are in
\Cref{app:ablations}.

\begin{figure}[t]
  \centering
  \begin{subfigure}[t]{0.24\textwidth}
    \includegraphics[width=\textwidth]{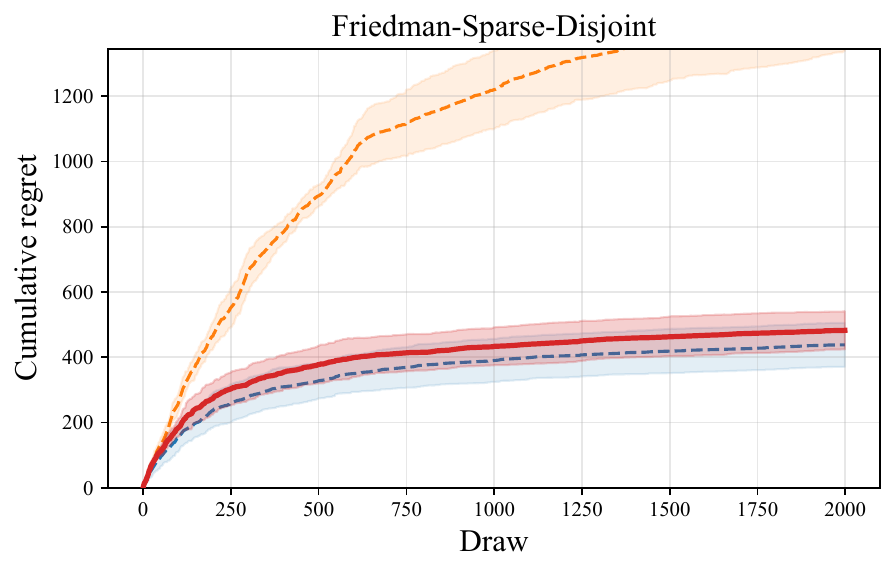}
    \caption{Sparse-disjoint}
  \end{subfigure}\hfill
  \begin{subfigure}[t]{0.24\textwidth}
    \includegraphics[width=\textwidth]{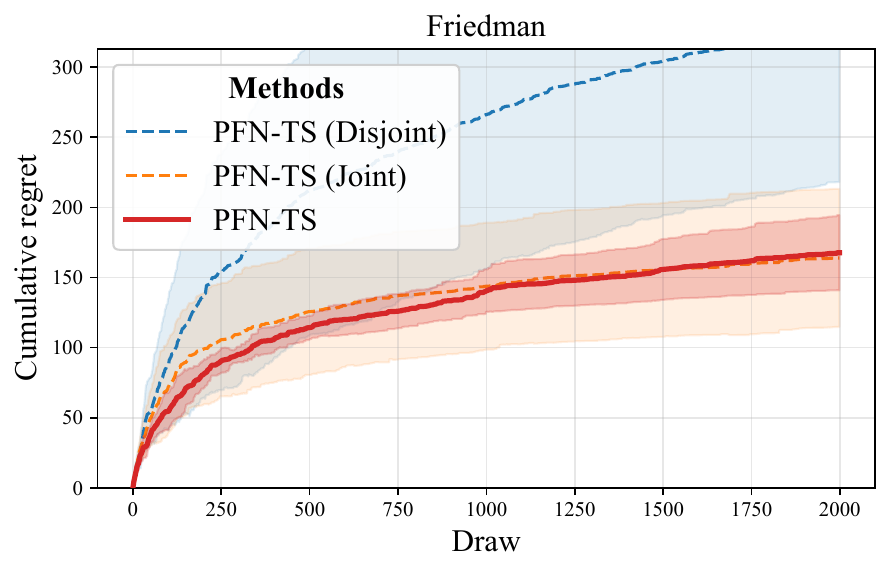}
    \caption{Friedman}
  \end{subfigure}\hfill
  \begin{subfigure}[t]{0.24\textwidth}
    \includegraphics[width=\textwidth]{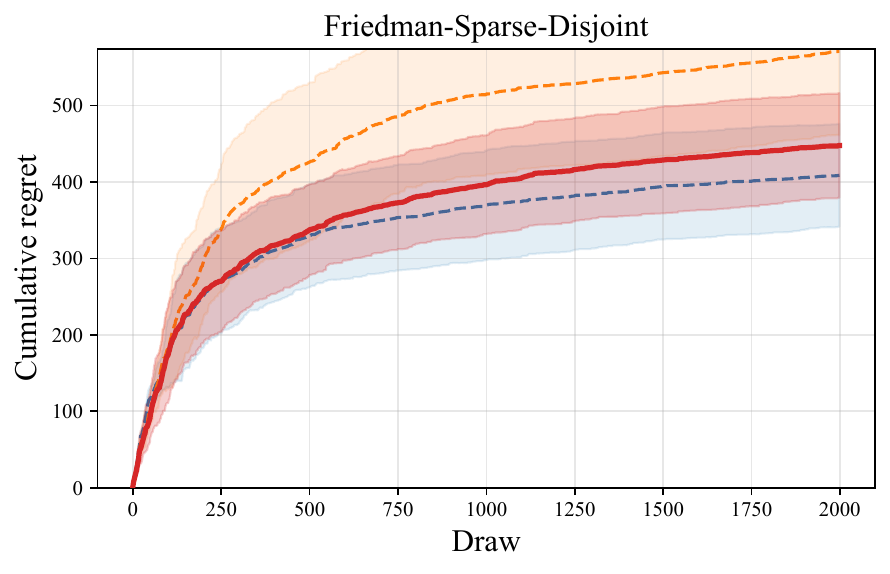}
    \caption{Sparse-disjoint}
  \end{subfigure}\hfill
  \begin{subfigure}[t]{0.24\textwidth}
    \includegraphics[width=\textwidth]{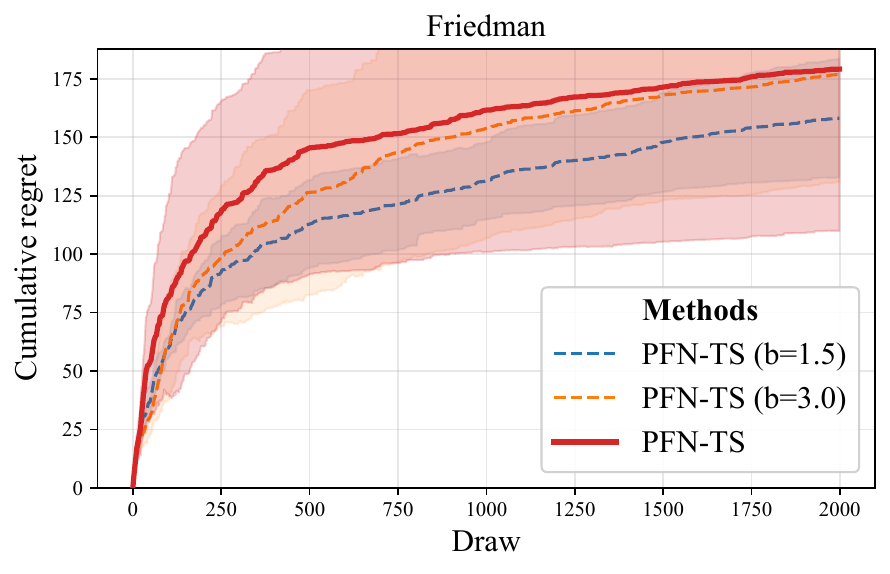}
    \caption{Friedman}
  \end{subfigure}
  \caption{Selected ablations on two representative nonlinear DGPs (mean cumulative
  regret, $R = 5$ replications). The first two panels vary the arm encoding, comparing
  the adaptive rule with fixed disjoint and joint encodings; the last two panels vary the
  geometric subsampling base $b \in \{1.5,2,3\}$.}
  \label{fig:main_ablations}
\end{figure}

\textbf{Encoding strategy}.
We compare the adaptive encoding rule (\Cref{sec:method}) against two fixed alternatives:
\textit{Disjoint} (separate context per arm, no arm indicator) and \textit{Joint} (arm
index appended to context as a categorical feature). 
Cumulative regret is sensitive to this
choice: disjoint encoding is poor when arms share structure, while joint encoding can be
less effective when arm-specific histories should remain separate. \Cref{fig:main_ablations} (left) shows that the adaptive rule is competitive with the better fixed encoding on both a shared-structure Friedman benchmark and a sparse disjoint-arm variant.

\textbf{Subsampling grid}.
We vary the geometric base $b \in \{1.5, 2, 3\}$. 
\Cref{fig:main_ablations} (right) shows that smaller bases can modestly improve regret
on some scenarios by using a denser prefix grid, but they require more forward passes. 
The default $b = 2$ gives a stable middle ground, with regret curves generally of the same
order as the neighboring choices while using half as many CLT forward passes as $b = 1.5$.

\textbf{Decision rule}.
We compare Thompson sampling (the default) against \textit{PFN-PS} (sampling from the
posterior predictive distribution) and \textit{PFN-Greedy} (selecting the arm with the
highest predictive mean). PFN-PS performs substantially worse, confirming that sampling
noisy future rewards from the PPD is not a substitute for sampling latent mean rewards.
Greedy selection can be competitive when the reward landscape is identified quickly, but
PFN-TS retains an edge on harder tasks.

\section{Discussion}
\label{sec:discussion}

PFN-TS shows that tabular in-context learning can serve as an effective reward model for
sequential decision-making once efficient uncertainty extraction and adaptive encoding are
in place.

\textbf{Computation.}
Runtime is the main current limitation: SubCLT reduces uncertainty-estimation cost, but
each decision still requires repeated TabICL evaluations across arms and grid points.
Preliminary experiments suggest that reducing TabICL's ensemble size from
$n_{\mathrm{estimators}}=8$ to $1$ substantially improves wall-clock time without
noticeably changing regret, and low GPU utilization suggests room for batching,
data-loading, and inference optimizations.

\textbf{Finite-sample calibration.}
SubCLT is asymptotic and can be poorly calibrated when arm histories are short or weakly
informative, as in the small-$n$ coverage diagnostics, early Drink Less horizons, and
Shuttle, where most arms observe zero rewards for many rounds. Regularized variants, such
as shrinkage toward pooled variances, adaptive variance floors, or prior-informed
pseudo-observations, are a natural next step.

\textbf{Extensions.}
Posterior contraction or information-gain bounds for the TabPFN/TabICL prior would sharpen
our Bayesian regret guarantee and could imply frequentist regret bounds. The adaptive
encoding rule is empirical, large or infinite action spaces would require candidate
generation or action search, and SubCLT-style uncertainty estimates may also be useful in
active learning, adaptive experimentation, and online model selection with tabular
foundation models.

\textbf{Deployment.}
PFN-TS should not be deployed in high-stakes adaptive decision systems without
application-specific validation. In settings such as mobile health, posterior uncertainty
can improve exploration, but policies should be monitored for subgroup performance, safety
constraints, and feedback loops induced by adaptive data collection.

\begin{ack}
QZ is supported by the National Key R\&D Program of China Grant 2024YFA1015800.
BC is supported by MOE Tier 2 grant MOE-T2EP20122-0013.
YT is supported by NUS Start-up Grant A-8000448-00-00 and MOE AcRF Tier 1 Grant A-8002498-00-00.
Additionally, BC and YT are supported by MOE AcRF Tier 1 Grant A-8004458-00-00.
\end{ack}

{\small
\bibliographystyle{plainnat}
\bibliography{references}
}

\appendix
\section{Proofs}
\label{app:proofs}

\subsection{Proof of \Cref{thm:clt}}
\label{app:proof_clt}

\begin{proof}
Write $\delta_i = m_i(x) - m_{i-1}(x)$ and decompose each increment as
\[
  \delta_i = \zeta_i + \xi_i,
\]
where $\zeta_i = \delta_i - \E[\delta_i \mid \cF_{i-1}]$ is the martingale-difference part
and $\xi_i = \E[\delta_i \mid \cF_{i-1}]$ is the conditional-mean (bias) part.

\medskip
\noindent\textbf{Part 1: CLT for $\sqrt{n}(m_\infty(x) - m_n(x))$.}

We apply the quasi-martingale CLT of \citet{berti2011central} (Proposition~1, restated as
Theorem~G.4 in \citet{fortini2026uncertainty}) to the scalar sequence $(m_n(x))$.
We verify its four requirements.

\emph{(UI) Uniform integrability of $(m_n(x))$.}
By conditional Jensen, $\E[\delta_i^2 \mid \cF_{i-1}]^2 \leq \E[\delta_i^4 \mid \cF_{i-1}]$,
so $\E\bigl[(i^2\,\E[\delta_i^2 \mid \cF_{i-1}])^2\bigr] \leq i^4\,\E[\delta_i^4] \leq C$
by \Cref{assump:moments}. Hence $\{i^2\,\E[\delta_i^2 \mid \cF_{i-1}]\}$ is $L^2$-bounded
and therefore uniformly integrable, which allows taking expectations through the a.s.\ limit
in \Cref{assump:covariance}: $\E[\delta_i^2] \leq (V(x)+1)/i^2$ for all large $i$, and
$\sum_{i\geq 1}\E[\delta_i^2] < \infty$. The martingale part $M_n = \sum_{i=1}^n \zeta_i$
is therefore $L^2$-bounded. The bias part $B_n = \sum_{i=1}^n \xi_i$ converges
absolutely in $L^2$ because $\sum_i\|\xi_i\|_2 \leq
\sum_i\sqrt{i}\,\|\xi_i\|_2 < \infty$ by \Cref{assump:quasi_martingale}. Hence
$(m_n(x))$ is $L^2$-bounded and uniformly integrable, and
$m_\infty(x) = m_0(x) + \sum_{i=1}^\infty \delta_i$ exists a.s.

\emph{(i) Quasi-martingale condition with $\sqrt{i}$ weights.}
Theorem~G.4 requires $\sum_{i\geq 1}\sqrt{i}\,\E[|\xi_i|] < \infty$.
Since $\E[|\xi_i|] \leq \|\xi_i\|_2$ by Jensen's inequality, this follows immediately from
\Cref{assump:quasi_martingale}.

\emph{(ii) Supremum bound.}
Since $\bigl(\sup_{i\geq 1}\sqrt{i}\,|\delta_i|\bigr)^4 \leq \sum_{i=1}^\infty i^2\delta_i^4$,
taking expectations and applying \Cref{assump:moments}:
\[
  \E\!\left[\Bigl(\sup_{i\geq 1}\sqrt{i}\,|\delta_i|\Bigr)^{\!4}\right]
  \leq \sum_{i=1}^\infty i^2\,\E[\delta_i^4]
  \leq C\sum_{i=1}^\infty i^{-2} < \infty,
\]
so $\sup_i \sqrt{i}\,|\delta_i| \in L^4 \subset L^1$, verifying condition~(ii).

\emph{(iii) A.s.\ variance convergence.}
Apply \citet{fortini2026uncertainty}, Lemma~G.5 with $W_i = i^2\delta_i^2$. The two
lemma conditions are:
(a)~$\sum_{i\geq 1}i^{-2}\E[W_i^2] = \sum_i i^2\,\E[\delta_i^4] \leq C\sum_i i^{-2}
< \infty$ by \Cref{assump:moments};
(b)~$\E[W_{i+1}\mid\cF_i] = (i{+}1)^2\,\E[\delta_{i+1}^2\mid\cF_i] \to V(x)$ a.s.\
by \Cref{assump:covariance}.
The lemma then gives $n\sum_{i\geq n} W_i / i^2
= n\sum_{i\geq n}\delta_i^2 \to V(x)$ a.s., which is condition~(iii).

All four requirements of Theorem~G.4 are satisfied, giving
$\cL(\sqrt{n}(m_\infty(x) - m_n(x)) \mid \cD_{1:n}) \xrightarrow{w} \cN(0,V(x))$,
$\mathbb{P}$-a.s.

\medskip
\noindent\textbf{Part 2: Consistency of $\hat{V}_{s_n}(x)$.}

Write $D_j = M_j + R_j$ where $M_j = \sum_{i=t_{j-1}+1}^{t_j} \zeta_i$ and
$R_j = \sum_{i=t_{j-1}+1}^{t_j} \xi_i$.

\noindent\textit{Step 2a: Block mean.}
Let $I_j=\{t_{j-1}+1,\ldots,t_j\}$. On the geometric grid, for all large $j$,
$t_j/t_{j-1}\le b$ and $t_j-t_{j-1}\ge c_b t_{j-1}$ for a constant $c_b>0$. Hence
$w_j=O(t_{j-1})$ and
\[
  \sum_{i\in I_j} i^{-2}=w_j^{-1}(1+o(1)).
\]

We first identify the total predictable second moment of the block. Put
$A_i=i^2\E[\delta_i^2\mid\cF_{i-1}]$. By \Cref{assump:covariance},
$A_i\to V(x)$ a.s.; by conditional Jensen and \Cref{assump:moments}, $\{A_i^2\}$ is
uniformly integrable. Thus $A_i\to V(x)$ in $L^2$, and Toeplitz' lemma gives
\[
  w_j\sum_{i\in I_j} \frac{A_i}{i^2} \to V(x)
  \quad\text{in }L^2.
\]
Taking expectations, this gives
$w_j\sum_{i\in I_j}\E[\delta_i^2]\to V(x)$.

We now pass from the total increment $\delta_i$ to its martingale part. The decomposition
$\delta_i=\zeta_i+\xi_i$ is orthogonal in $L^2$, so
$\E[\delta_i^2]=\E[\zeta_i^2]+\E[\xi_i^2]$. Let
$T_j=\sum_{i\ge t_{j-1}}\sqrt{i}\,\|\xi_i\|_2$; by
\Cref{assump:quasi_martingale}, $T_j\to0$. Then
\[
  \|R_j\|_2
  \leq \sum_{i\in I_j}\|\xi_i\|_2
  \leq \frac{T_j}{\sqrt{t_{j-1}}},
  \qquad
  w_j\sum_{i\in I_j}\E[\xi_i^2]
  \leq \frac{w_j}{t_{j-1}}T_j^2=o(1).
\]
Consequently
\[
  w_j\E[M_j^2]
  = w_j\sum_{i\in I_j}\E[\zeta_i^2]
  = w_j\sum_{i\in I_j}\E[\delta_i^2] - w_j\sum_{i\in I_j}\E[\xi_i^2]
  \to V(x),
\]
and $\|M_j\|_2=O(w_j^{-1/2})$. Finally, expanding
$D_j^2=(M_j+R_j)^2$ shows that the remaining drift terms vanish:
\[
  2w_j\E|M_jR_j|
  \leq 2w_j\|M_j\|_2\|R_j\|_2
  = O(T_j)\to0,
\]
and
$w_j\E[(R_j)^2]\leq (w_j/t_{j-1})T_j^2=o(1)$. Therefore
$w_j\E[D_j^2]\to V(x)$.

\noindent\textit{Step 2b: Second-moment bound for a single block.}
By the martingale Burkholder--Rosenthal inequality,
\[
  \E[M_j^4] \leq C_R\!\left(
    \E\!\left[\left(\sum_{i=t_{j-1}+1}^{t_j}
      \E[\zeta_i^2\mid\cF_{i-1}]\right)^{\!2}\right]
    + \sum_{i=t_{j-1}+1}^{t_j}\!\E[\zeta_i^4]\right)
\]
The predictable-variation term is $O(t_{j-1}^{-2})$ because
$\E[\zeta_i^2\mid\cF_{i-1}]\leq \E[\delta_i^2\mid\cF_{i-1}]$ and the weighted block
averages above are $L^2$-bounded. The fourth-moment term is
$O(t_{j-1}^{-3})$ by \Cref{assump:moments}. Therefore
$\E[M_j^4]=O(t_{j-1}^{-2})$. Since $w_j = O(t_{j-1})$, this gives
$\E[(w_j M_j^2)^2] = O(1)$. Individual blocks do not concentrate to $V(x)$; averaging
is essential.

\noindent\textit{Step 2c: Averaging over $J$ blocks.}
We first remove the drift contribution in
$D_j^2=(M_j+R_j)^2$. Step~2a gives
$w_j\E|M_jR_j|=O(T_j)$ and $w_j\E[(R_j)^2]=O(T_j^2)$, so Ces\`aro's lemma gives
\[
  \frac{1}{J}\sum_{j=1}^J
  w_j\bigl(D_j^2-M_j^2\bigr)
  \xrightarrow{L^1}0.
\]
It remains to average $w_jM_j^2$. Let
\[
  Y_j=\E[w_jM_j^2\mid\cF_{t_{j-1}}],
  \qquad
  U_j=w_jM_j^2-Y_j.
\]
Then $\{U_j\}$ is a martingale difference sequence with respect to the block filtration
$\{\cF_{t_j}\}$. By Step~2b,
\[
  \E\!\left[\left(\frac{1}{J}\sum_{j=1}^J U_j\right)^2\right]
  = \frac{1}{J^2}\sum_{j=1}^J \E[U_j^2]
  = O(1/J)\to0.
\]

Finally, we control the predictable part. Since martingale cross terms vanish
conditionally,
\[
  Y_j
  = \E\!\left[
      w_j\sum_{i\in I_j}\E[\zeta_i^2\mid\cF_{i-1}]
      \,\middle|\,\cF_{t_{j-1}}
    \right].
\]
Using
$\E[\zeta_i^2\mid\cF_{i-1}]
=A_i/i^2-\xi_i^2$, the contraction property of conditional expectation, the
$L^2$ convergence
$w_j\sum_{i\in I_j}A_i/i^2\to V(x)$, and the bound
$w_j\sum_{i\in I_j}\E[\xi_i^2]\to0$ from Step~2a, we obtain
$Y_j\to V(x)$ in $L^1$. Hence, again by Ces\`aro's lemma,
\[
  \E\!\left[\left|\frac{1}{J}\sum_{j=1}^J (Y_j-V(x))\right|\right]\to0.
\]
Combining the drift removal, martingale-difference average, and predictable average gives
$\hat{V}_{s_n}(x)\xrightarrow{p} V(x)$.
\end{proof}

\begin{remark}[On the moment condition]
\label{rem:moment_condition}
\Cref{assump:moments} is a convenient sufficient condition rather than a tight
requirement. The CLT itself can be proved under weaker Lindeberg-type or maximal-increment
conditions, and the variance-estimator argument mainly needs enough uniform integrability
to control the weighted block squares $w_jD_j^2$. We state the stronger
$4+\epsilon$ moment condition because it gives a short proof of both requirements while
keeping the theorem aligned with the predictive-CLT assumptions used above.
\end{remark}

\subsection{Proof of \Cref{thm:main}}
\label{app:proof_main}

\begin{proof}
Construct PFN-TS and exact $\Pi$-TS on the same probability space. They share the same
draw $f_0\sim\Pi$, the same context sequence, and the same reward noise whenever they
select the same action. If their histories agree at the start of round $t$, couple their
actions by a maximal coupling of
$Q_t(\cdot\mid X_t,\cD_{t-1}^{(1:K)})$ and
$P_t^\Pi(\cdot\mid X_t,\cD_{t-1}^{(1:K)})$, where
$\cD_{t-1}^{(1:K)}$ denotes the collection of arm-specific histories. Then the
conditional probability that their actions disagree is at most $\varepsilon_t$.

Let $\tau$ be the first round in which the coupled actions disagree, with $\tau=\infty$
if this never happens. If $\tau>t$, the two algorithms have identical histories and take
the same action through round $t$. Thus
\[
  \Pr(\tau=t\mid \tau\geq t)\leq \varepsilon_t,
  \qquad
  \Pr(\tau\leq t)\leq \sum_{s=1}^t \varepsilon_s.
\]
The regret difference at round $t$ is zero on $\{\tau>t\}$. On $\{\tau\leq t\}$, condition
on the context, the two actions, and the combined arm-specific histories observed by the
coupled algorithms before the reward at round $t$. The posterior expected reward-range
condition in \Cref{thm:main} gives
\[
  \E\!\left[
    \left|f_0(X_t,A_t^{\Pi\text{-}\mathrm{TS}})
    -f_0(X_t,A_t^{\mathrm{PFN}})\right|
    \,\middle|\,
    X_t,A_t^{\Pi\text{-}\mathrm{TS}},A_t^{\mathrm{PFN}},
    \cD_{t-1}^{\mathrm{cpl}}
  \right]
  \leq 2B_R,
\]
where $\cD_{t-1}^{\mathrm{cpl}}$ denotes the combined histories generated by the coupled
process. Hence
\[
  \E[\Regret_T^{\mathrm{PFN}}]
  \leq
  \E[\Regret_T^{\Pi\text{-}\mathrm{TS}}]
  + 2B_R\sum_{t=1}^T \Pr(\tau\leq t).
\]
Using the previous union bound and exchanging the order of summation gives
\[
  \E[\Regret_T^{\mathrm{PFN}}]
  \leq
  \E[\Regret_T^{\Pi\text{-}\mathrm{TS}}]
  + 2B_R\sum_{t=1}^T\sum_{s=1}^t \varepsilon_s
  =
  \E[\Regret_T^{\Pi\text{-}\mathrm{TS}}]
  + 2B_R\sum_{s=1}^T (T-s+1)\varepsilon_s.
\]
\end{proof}

\begin{remark}[Information-gain bound for exact $\Pi$-TS]
The exact Bayesian term in \Cref{thm:main} can be bounded by applying the standard
information-ratio inequality to exact $\Pi$-TS:
\[
  \E[\Regret_T^{\Pi\text{-}\mathrm{TS}}]
  \leq
  \sqrt{2T\,\bar{\Gamma}_T(\Pi)\,I_\Pi(f_0;\cH_T^{\Pi\text{-}\mathrm{TS}})}.
\]
Under disjoint encoding and the product prior $\Pi=\bigotimes_{a=1}^K\Pi_a$, the
information acquired across arms factorizes as
\[
  I_\Pi(f_0;\cH_T)
  =
  \sum_{a=1}^K I_\Pi(f_0(\cdot,a);\cH_{T,a}).
\]
Consequently,
$I_\Pi(f_0;\cH_T^{\Pi\text{-}\mathrm{TS}})\leq \gamma_T(\Pi)$, where
\[
  \gamma_n(\Pi_a)
  =
  \sup_{x_1,\ldots,x_n\in\cX}
  I_{\Pi_a}\!\left(f_0(\cdot,a);(R_1,\ldots,R_n)\mid X_i=x_i,\;A_i=a\right)
\]
is the usual maximum information gain for arm $a$, and
\[
  \gamma_T(\Pi)
  =
  \max_{n_1+\cdots+n_K=T}
  \sum_{a=1}^K \gamma_{n_a}(\Pi_a).
\]
If the posterior reward distributions under $\Pi$ are uniformly sub-Gaussian with proxy
$\sigma_R^2$, finite-action Thompson Sampling information-ratio bounds suggest
$\bar{\Gamma}_T(\Pi)\lesssim K\sigma_R^2$ \citep{gouverneur2023thompson}. We keep
$\bar{\Gamma}_T(\Pi)$ explicit because this is a condition on Bayesian reward
distributions, not merely on observation noise.
\end{remark}

\begin{proposition}[Sampler error from PPD approximation]
\label{prop:sampler_error}
For each arm $a$, let $\cD^{(a)}$ denote the arm-specific history used by disjoint
encoding, let $s_a$ be its latest refresh size, and write
$\cD_s^{(a)}=\cD^{(a)}_{1:s_a}$. Let $m^\Pi(x;\cD_s^{(a)})$ be the exact posterior
predictive mean under $\Pi$, and let $m(x;\cD_s^{(a)})$ be the predictive mean used by
PFN-TS. Let $\hat{V}^\Pi(x;\cD_s^{(a)})$ be the SubCLT variance estimate obtained by
applying SubCLT to the exact $\Pi$-predictive mean sequence, and let
$\hat{V}(x;\cD_s^{(a)})$ be the SubCLT variance estimate from the TabICL predictive
sequence. Assume these variance estimates are strictly positive; in applications this can
be enforced by applying a common deterministic variance floor. For
$\cD^{(1:K)}=(\cD^{(1)},\ldots,\cD^{(K)})$, define
\begin{align*}
  G_t^\Pi(x,\cD^{(1:K)})
  &=
  \bigotimes_{a=1}^K
  \cN\!\left(
    m^\Pi(x;\cD_s^{(a)}),
    \frac{\hat{V}^\Pi(x;\cD_s^{(a)})}{s_a}
  \right),\\
  G_t(x,\cD^{(1:K)})
  &=
  \bigotimes_{a=1}^K
  \cN\!\left(
    m(x;\cD_s^{(a)}),
    \frac{\hat{V}(x;\cD_s^{(a)})}{s_a}
  \right).
\end{align*}
Let $S_t^\Pi(\cdot\mid x,\cD^{(1:K)})$ be the exact posterior sampling distribution over
action-value vectors under $\Pi$. Suppose PFN-TS samples an action-value vector from
$G_t(x,\cD^{(1:K)})$ at round $t$, and let
\[
  \rho_t =
  \sup_{x,\cD^{(1:K)}}\TV\!\left(
    S_t^\Pi(\cdot\mid x,\cD^{(1:K)}),
    G_t^\Pi(x,\cD^{(1:K)})
  \right)
\]
denote the uniform Gaussian approximation error for exact posterior sampling under
$\Pi$, including the error from comparing exact full-history posterior sampling with the
snapshot posterior at refresh sizes $(s_1,\ldots,s_K)$. Define
\[
  \Delta_{t,a}(x,\cD^{(a)})
  =
  m(x;\cD_s^{(a)})-m^\Pi(x;\cD_s^{(a)}),
  \qquad
  r_{t,a}(x,\cD^{(a)})
  =
  \frac{\hat{V}(x;\cD_s^{(a)})}
       {\hat{V}^\Pi(x;\cD_s^{(a)})}.
\]
Then
\[
  \begin{aligned}
  \varepsilon_t
  \leq \rho_t
  + \frac{1}{2}\sup_{x,\cD^{(1:K)}}
    \left(
      \sum_{a=1}^K
      \left[
        r_{t,a}(x,\cD^{(a)})
        +\frac{s_a\Delta_{t,a}^2(x,\cD^{(a)})}
             {\hat{V}^\Pi(x;\cD_s^{(a)})}
        -1
        -\log r_{t,a}(x,\cD^{(a)})
      \right]
    \right)^{1/2}
  .
  \end{aligned}
  \label{eq:sampler_error_bound}
\]
\end{proposition}

\begin{proof}[Proof of \Cref{prop:sampler_error}]
The action is a measurable function of the sampled action-value vector, so data
processing gives
\[
  \TV\!\left(
    Q_t(\cdot\mid x,\cD^{(1:K)}),
    P_t^\Pi(\cdot\mid x,\cD^{(1:K)})
  \right)
  \leq
  \TV\!\left(
    G_t(x,\cD^{(1:K)}),
    S_t^\Pi(\cdot\mid x,\cD^{(1:K)})
  \right).
\]
By the triangle inequality for total variation,
\[
  \TV(G_t,S_t^\Pi)
  \leq
  \TV(G_t,G_t^\Pi)
  + \TV(G_t^\Pi,S_t^\Pi),
\]
where all terms are evaluated at a fixed $(x,\cD^{(1:K)})$. The last term is bounded by
$\rho_t$. It remains to bound the TV distance between the two diagonal Gaussians. By
Pinsker's inequality and the KL divergence between diagonal Gaussians,
\[
  \begin{aligned}
  \TV(G_t,G_t^\Pi)
  &\leq
  \sqrt{\frac{1}{2}\operatorname{KL}(G_t\,\|\,G_t^\Pi)}
  \\
  &=
  \frac{1}{2}
  \left(
    \sum_{a=1}^K
      \left[
        r_{t,a}(x,\cD^{(a)})
      +\frac{s_a\Delta_{t,a}^2(x,\cD^{(a)})}
           {\hat{V}^\Pi(x;\cD_s^{(a)})}
      -1
      -\log r_{t,a}(x,\cD^{(a)})
    \right]
  \right)^{1/2}.
  \end{aligned}
\]
Taking the supremum gives \Cref{eq:sampler_error_bound}.
\end{proof}

\begin{remark}[Interpreting the residual term]
The residual $\rho_t$ is the Gaussian approximation error for exact posterior sampling
under $\Pi$. When the predictive CLT gives a Berry--Esseen rate for the exact posterior,
one expects
$\rho_t=O(C_{\mathrm{BE}}/\sqrt{s_{t,\min}})$, where
$s_{t,\min}=\min_a s_{t,a}$ is the smallest refresh size, up to finite-grid SubCLT error
computed with the exact $\Pi$-predictive sequence and any posterior-refresh error from
using snapshot rather than full histories.
\end{remark}

\begin{remark}[Variance error from predictive-mean paths]
The variance term in \Cref{prop:sampler_error} can also be controlled by approximation error
of the predictive means along the SubCLT grid. Let $t_0<\cdots<t_J$ be the grid for an
arm history $\cD^{(a)}$, and write
\[
  D_j
  =
  m(x;\cD^{(a)}_{1:t_j})
  -
  m(x;\cD^{(a)}_{1:t_{j-1}}),
  \qquad
  D_j^\Pi
  =
  m^\Pi(x;\cD^{(a)}_{1:t_j})
  -
  m^\Pi(x;\cD^{(a)}_{1:t_{j-1}}).
\]
If $\hat{V}=J^{-1}\sum_{j=1}^Jw_jD_j^2$ and
$\hat{V}^{\Pi}=J^{-1}\sum_{j=1}^Jw_j(D_j^\Pi)^2$, then
\[
  |\hat{V}-\hat{V}^{\Pi}|
  \leq
  \frac{1}{J}\sum_{j=1}^J
  w_j
  |D_j-D_j^\Pi|
  \left(|D_j|+|D_j^\Pi|\right).
\]
Moreover, with
$e_i(x;\cD^{(a)})=
|m(x;\cD^{(a)}_{1:i})-m^\Pi(x;\cD^{(a)}_{1:i})|$,
\[
  |D_j-D_j^\Pi|
  \leq
  e_{t_j}(x;\cD^{(a)})+e_{t_{j-1}}(x;\cD^{(a)}).
\]
Thus the variance component is also controlled by TabICL's approximation to the exact
PPD predictive means, but along the full SubCLT prefix path rather than only at the final
history.
\end{remark}

\begin{remark}[Interpreting the information gain]
The information gain $\gamma_T(\Pi)$ for the prior $\Pi$ associated with TabICL is not
explicitly characterized. It should be read as the Bayesian complexity of the exact
prior model that PFN-TS is approximating. In contextual bandits, familiar examples give
the following scales, suppressing constants and noise parameters. With $K$ arms and
$d$-dimensional linear rewards, the disjoint information gain is
$\gamma_T=O(Kd\log T)$ \citep{agrawal2013thompson,russo2018tutorial}. For
squared-exponential kernel or Gaussian-process models in $d$ dimensions,
$\gamma_T=O(K(\log T)^{d+1})$ \citep{krause2011contextual,chowdhury2017kernelized}. For
Mat\'ern kernels with smoothness parameter $\nu$, a common GP proxy for
Sobolev/H\"older-type smoothness, the information gain is
$\gamma_T=\tilde{O}(K T^{d/(2\nu+d)})$ \citep{vakili2021information}; heuristically, an
$\alpha$-H\"older class gives the analogous rate
$\tilde{O}(K T^{d/(2\alpha+d)})$. For BART priors, the Bayesian regret bound of
\citet{deng2026bfts} is
$O(K\sqrt{T\log T})$, which corresponds to a logarithmic effective complexity
term, of order $K\log T$ in the information-gain part of the bound, for the BART model
class \citep{chipman2010bart,deng2026bfts}.

For TabICL and TabPFN, the prior $\Pi$ is meant to cover a broad collection of tabular
data-generating mechanisms rather than a single hand-specified class. Empirical evidence
that TabPFN adapts across many simple and complex tabular regimes
\citep{zhang2025tabpfnmodelruleall} suggests that the effective information gain of
$\Pi$ should also adapt to the structure revealed by the history. Thus, although a
worst-case bound for the full prior could be loose, one expects $\gamma_T(\Pi)$ to behave
more like the information gain of a simpler local model on easy instances, leading to
regret closer to the $\sqrt{T}$ rate when the realized reward functions have low effective
complexity. Proving such an adaptive information-gain bound for PFN priors is left for
future work.
\end{remark}

\section{Algorithm details}
\label{app:algorithm}

\subsection{Subsampled CLT}
\label{app:subclt_algorithm}

\begin{algorithm}[t]
\caption{Subsampled CLT (\textsc{SubCLT})}
\label{alg:subsampled_clt}
\begin{algorithmic}[1]
\REQUIRE Model $\cM$, history $\cD = \{(x_i, r_i)\}_{i=1}^n$, query $x$, base $b > 1$;
  the grid below contains at least one block
\STATE Compute grid $G = \{2 {=} t_0 < t_1 < \cdots < t_J \le n\}$ by iterating
  $t_{j+1} = \max(t_j{+}1,\, \lfloor b\,t_j \rfloor)$ while $t_{j+1} \le n$
\STATE $m_0 \leftarrow m(x;\, \cD_{1:2})$ \hfill $\triangleright$ forward pass on first 2 observations
\FOR{$j = 1, \ldots, J$}
  \STATE $m_j \leftarrow m(x;\, \cD_{1:t_j})$
    \hfill $\triangleright$ forward pass, reuse cache from step $j{-}1$
  \STATE $D_j \leftarrow m_j - m_{j-1}$
  \STATE $w_j \leftarrow t_j \, t_{j-1} / (t_j - t_{j-1})$
\ENDFOR
\STATE $\hat{V} \leftarrow \frac{1}{J} \sum_{j=1}^J w_j D_j^2$
\RETURN $m_J,\; \hat{V},\; t_J$ \hfill $\triangleright$ latest snapshot mean, variance, and size
\end{algorithmic}
\end{algorithm}

\subsection{Full algorithm}
\label{app:full_algorithm}

\begin{algorithm}[t]
\caption{PFN-TS}
\label{alg:tabicl_ts}
\begin{algorithmic}[1]
\REQUIRE Model $\cM$, horizon $T$, arms $K$, base $b$,
  switch times $S = \{s_1 < \cdots < s_L\}$, warm-up rounds $\tau$,
  arm threshold $K_{\mathrm{thr}}$
\STATE $e^* \leftarrow \text{disjoint}$ if $K < K_{\mathrm{thr}}$, else $\text{one\text{-}hot}$;
  \; $e^\dagger \leftarrow$ the other encoding
\STATE $C \leftarrow 0$, $C^\dagger \leftarrow 0$; dual-caching $\leftarrow$ \TRUE
\FOR{$t = 1, \ldots, T$}
  \STATE Observe $X_t$
  \IF{$t \leq \tau K$}
    \STATE $A_t \leftarrow 1 + ((t-1) \bmod K)$ \hfill $\triangleright$ round-robin warm-up
  \ELSE
    \FOR{$k = 1, \ldots, K$}
      \STATE $(\bar m, \hat{V}, s_k) \leftarrow
        \textproc{SubCLT}(\cM, \cD^{(k,e^*)}, X_t, b)$
      \STATE $\tilde{r}_k \sim \cN\!\left(\bar m,\; \hat{V} / s_k\right)$
    \ENDFOR
    \STATE $A_t \leftarrow \argmax_k \tilde{r}_k$
  \ENDIF
  \STATE Observe $R_t$; append $(X_t, R_t)$ to active history $\cD^{(A_t,e^*)}$
  \IF{dual-caching}
    \STATE Also append $(X_t, R_t)$ to challenger history $\cD^{(A_t,e^\dagger)}$;
      extend challenger snapshot cache if arm grid point reached
  \ENDIF
  \IF{$t \in S$}
    \STATE Score all obs since last switch for each encoding using cached
      snapshots; add interval CRPS to $C$, $C^\dagger$ respectively
    \IF{$C^\dagger < C$}
      \STATE Swap $e^* \leftrightarrow e^\dagger$; swap all caches and histories;
        swap $C \leftrightarrow C^\dagger$
    \ENDIF
    \IF{$t = s_L$}
      \STATE dual-caching $\leftarrow$ \FALSE; discard challenger caches
    \ENDIF
  \ENDIF
\ENDFOR
\end{algorithmic}
\end{algorithm}

\subsection{Encoding selection and termination}
\label{app:early_termination}

At each switch time $s \in S$, the CRPS for all observations since the previous switch
time is computed for both encodings by replaying observations against the cached model
snapshots; this interval CRPS is added to the respective cumulative totals $C$ and
$C^\dagger$. The encoding with lower cumulative CRPS becomes the active encoding.
Dual caching terminates unconditionally at the last switch time $s_L$ (default
$s_L = 2048$), after which only the active encoding's caches are retained and the
challenger is discarded. No additional forward passes are required for CRPS scoring
since the snapshot caches built during Thompson Sampling are reused.

\subsection{CRPS computation}
\label{app:crps}

For TabICL's discretized output distribution (a probability mass function over binned
response values), CRPS is computed exactly as a sum over the bin CDF:
\[
  \CRPS(\hat{F}, r) = \sum_{j} (\hat{F}(y_j) - \mathbf{1}\{y_j \geq r\})^2 \Delta y_j,
\]
where $y_j$ are the bin midpoints and $\Delta y_j$ are the bin widths. This is computed
at zero additional forward-pass cost since $\hat{F}$ is already available from the
prediction step.

\subsection{Caching in TabICL}
\label{app:caching}

TabICL caches the key-value representations for a fixed training set $\cD_{1:t_j}$ after
computing $\mu_{t_j}(x)$. Because the grid is evaluated in increasing order, the cache
from step $j-1$ can be extended to step $j$ by appending the new observations
$\cD_{t_{j-1}+1:t_j}$ and recomputing only the affected attention layers. This reduces
the per-step cost from $O(t_j^2)$ to $O((t_j - t_{j-1}) \cdot t_j)$, a substantial
saving on the coarser geometric grid.

\section{Experimental Specifications}
\label{app:experiments}

All experiments use global random seed 42, with per-replication seeds derived from it.

\paragraph{Compute environment.}
PFN-TS used an NVIDIA RTX 5090 GPU for the synthetic experiments, ablation studies, and
coverage analysis, and an NVIDIA RTX PRO 6000 GPU with 96 GB of memory for the OpenML
experiments. 
On each dataset, wall-clock time for a $T=10,000$ horizon is a few hours.
NeuralTS \citep{zhang2021neural} used an NVIDIA A40 GPU and one 36-core CPU
socket. All other baselines ran on CPU with 4 logical cores. BFTS used 4 parallel MCMC
chains, and RFTS and XGBoostTS (TETS-RF and TETS-XGBoost in \citealp{nilsson2024tree})
used OpenMP parallelization with \texttt{nthread=4}. Because hardware and parallelization
differed across methods, we do not report detailed wall-clock runtime comparisons.

\subsection{Synthetic DGP Specifications}
\label{app:dgp}

All synthetic experiments use horizon $T = 10{,}000$ and $R = 5$ independent replications.
The interaction proceeds at each round $t = 1, \ldots, T$ as follows: a context
$X_t \in \mathbb{R}^P$ arrives, the agent selects arm $a_t \in \{1, \ldots, K\}$, and
observes reward $Y_t = \mu_{a_t}(X_t) + \varepsilon_t$ with independent noise. Unless
stated otherwise, $\varepsilon_t \sim \mathcal{N}(0, \sigma^2)$.

\paragraph{Linear ($P = 10,\ K = 3$).}
Contexts are i.i.d.\ $X_t \sim \mathrm{Unif}([0,1]^P)$. For each arm $a$, coefficients
$\beta_a \in \mathbb{R}^P$ are drawn with $\beta_{a,j} \overset{\mathrm{iid}}{\sim}
\mathcal{N}(0,1)$. The mean reward is $\mu_a(x) = \beta_a^\top x$.

\paragraph{Friedman variants ($P \in \{5, 20\},\ K = 2$).}
Contexts are i.i.d.\ $X_t \sim \mathrm{Unif}([0,1]^P)$.
The base Friedman1 function is
\[
  f(x_1,\ldots,x_5)
    = 10\sin(\pi x_1 x_2) + 20(x_3 - 0.5)^2 + 10 x_4 + 5 x_5.
\]
For \textsc{Friedman2} and \textsc{Friedman3}, the first four coordinates are rescaled as
$x_1' = 100 x_1$, $x_2' = 40\pi + 520\pi x_2$, $x_3' = x_3$, $x_4' = 1 + 10 x_4$, with
\[
  f_2(x) = \tfrac{1}{125}\sqrt{(x_1')^2 + \bigl(x_2' x_3' - (x_2' x_4')^{-1}\bigr)^2},
  \qquad
  f_3(x) = \tfrac{1}{0.1}\arctan\!\Bigl(\tfrac{x_2' x_3' - (x_2' x_4')^{-1}}{x_1'}\Bigr).
\]
Arm 1 uses the scenario-specific function ($f$, $f_2$, or $f_3$). Arm 2 uses a
\emph{shared} variant $\mu_2(x) = f(x_1,\ldots,x_5) + 5\sin(\pi x_1 x_2)$, or a
\emph{disjoint} variant $\mu_2(x) = f(x_P, x_{P-1}, \ldots, x_1)$ (reversed feature order).
Sparse variants increase context dimensionality to $P = 20$ while keeping the reward
functions low-dimensional.

\paragraph{Friedman-Heteroscedastic.}
Same mean functions as the shared Friedman variant, with arm-specific noise: for each arm
$a$, $\sigma_a^2 = 10^{U_a}$ where $U_a \sim \mathrm{Unif}(-1, 1)$ is drawn once per
replication.

\paragraph{SynBART ($P = 4,\ K = 3$).}
Contexts are i.i.d.\ $X_t \sim \mathrm{Unif}([0,1]^P)$. Each arm's mean reward function
$\mu_a(\cdot)$ is sampled from a BART prior \citep{chipman2010bart} with $m = 100$ trees,
depth-geometric splitting prior with $\alpha = 0.45$, Gaussian leaf prior with $\kappa = 2$,
and Dirichlet-sparse split probabilities with $(\zeta, \xi) = (1, 1)$. Noise variance is
fixed to $\sigma^2 = 0.01$. This DGP is correctly specified for BFTS and represents a
challenging nonlinear benchmark for methods with a different prior.

\subsection{OpenML Dataset Details}
\label{app:openml}

We evaluate on eight OpenML classification datasets: Adult, Covertype, EEGEyeState, GasDrift,
MagicTelescope, Mushroom, PageBlocks, and Shuttle. Each dataset defines a contextual
bandit: the feature vector is the context, each class label is an arm, and the reward is 1
for selecting the true class and 0 otherwise. We follow the preprocessing of
\citet{zhang2021neural,deng2026bfts}: features are standardized, categorical variables are
one-hot encoded, and the interaction horizon is capped at $T = 10{,}000$ (or the dataset
size if smaller; Mushroom has 8{,}124 rows). LinTS and NeuralTS hyperparameters are
selected on a separate low-budget pilot run (one replication with a different seed) to
avoid post-hoc tuning.

\begin{remark}
Pilot tuning gives LinTS and NeuralTS a limited offline tuning budget relative to a
strictly online setting. PFN-TS uses the same default settings across datasets.
\end{remark}

\subsection{Drink Less Data Processing}
\label{app:drinkless}

The Drink Less data come from a micro-randomized trial (MRT) \citep{bell2020drinkless}
involving $n = 349$ participants over 30 days. At 8:00~p.m.\ daily (local time),
participants were randomized among three push-notification actions with static propensities
$(0.4, 0.3, 0.3)$: (i)~no message, (ii)~a standard message, and (iii)~a message drawn from
a bank of tailored variants. The binary reward $R_{i,d} \in \{0,1\}$ indicates proximal
engagement: whether participant $i$ opened the app within one hour of the 8~p.m.\ decision
point on day $d$.

The panel $\{(X_{i,d}, A_{i,d}, R_{i,d})\}$ is unfolded into a sequential stream of
$T = 349 \times 30 = 10{,}470$ decision points, ordered by day $d$ within each
participant $i$ and then interleaved across participants. This ordering preserves
within-participant temporal structure while allowing sequential simulation of adaptive
algorithms. The context $X_t$ includes static participant covariates (\texttt{AUDIT\_score},
\texttt{age}, \texttt{gender}, \texttt{employment\_type}), the time-varying covariate
\texttt{days\_since\_download}, and a one-hot user identifier to capture individual
fixed effects. Features are standardized prior to bandit simulation.

\paragraph{OPE methodology.}
Since the algorithms cannot be deployed prospectively, we use offline replay-style
sequential simulation \citep{li_unbiased_2011}: at each step $t$, the algorithm proposes an
action and is updated only if its proposal matches the logged action $A_t$, effectively
using the logged data as an unbiased estimator of the interactive bandit environment under
the logging propensities. Policy value is estimated via self-normalized importance
sampling (SNIPS;~\citet{swaminathan2015self}):
\[
  \widehat{V}_{\mathrm{SNIPS}}(\pi)
    = \frac{\sum_t w_t(\pi)\, R_t}{\sum_t w_t(\pi)},
  \qquad
  w_t(\pi) = \frac{\pi(A_t \mid X_t)}{\pi_0(A_t)},
\]
where $\pi_0(a) \in \{0.4, 0.3, 0.3\}$ are the known logging propensities.
Uncertainty is quantified via a user-level cluster bootstrap with $B = 30$ replicates.
A doubly-robust (DR) estimator \citep{dudik2011doubly} is provided as a robustness check
(see \Cref{app:ope_diagnostics}).

Since the logging propensities are bounded away from zero, importance weights are bounded
by $1/\min_a \pi_0(a) = 10/3$, ruling out severe positivity violations.
Weight diagnostics are reported in \Cref{app:ope_diagnostics}.

\subsection{Hyperparameter Settings}
\label{app:hyperparams}

\paragraph{PFN-TS.}
Unless stated otherwise in the ablation studies, PFN-TS uses the following defaults across
all experiments:
\begin{itemize}
  \item \textbf{Subsampling base}: $b = 2$ (geometric grid).
  \item \textbf{Context window}: no truncation; the full observed history for the relevant
    arm/encoding is retained.
  \item \textbf{Warm-up}: $\tau = 5$ round-robin pulls per arm before Thompson sampling.
  \item \textbf{Backbone model}: TabICL \citep{qu2026tabiclv2} with KV-cache enabled for
    the fixed training set per arm.
  \item \textbf{Encoding}: adaptive (disjoint for $n \leq N_\mathrm{thr}$, joint otherwise,
    with $N_\mathrm{thr} = 128$).
\end{itemize}

\paragraph{Baselines.}
Baseline implementations and fixed settings follow \citet{deng2026bfts} unless noted
below.
\begin{itemize}
  \item \textbf{LinTS}: exploration parameter $\nu$ tuned from $\{1.0, 0.1, 0.01\}$ on a
    pilot run; $\lambda = 1$ (fixed).
  \item \textbf{LinUCB}: exploration parameter $\alpha = 1$ (fixed).
  \item \textbf{NeuralTS}: $\lambda$ tuned from $\{1, 0.1, 0.01, 0.001\}$ and $\nu$ from
    $\{0.1, 0.01, 0.001, 0.0001\}$ on a pilot run.
  \item \textbf{BFTS}: default settings of \citep{deng2026bfts} with $m = 100$ trees per arm, depth-exponential prior with $\alpha = 0.45$, Dirichlet-sparse split
    probabilities, $n_{\mathrm{post}} = 500$, $n_{\mathrm{burn}} = 500$, 4 chains,
    logarithmic refresh schedule, and the $\tau = 5$ warm-up convention.
  \item \textbf{RFTS, XGBoostTS}: default settings of \citet{nilsson2024tree} with
    $\nu = 1$ (fixed).
\end{itemize}

\section{Additional Experimental Results}
\label{app:additional}

\subsection{Full Synthetic Regret Curves}
\label{app:full_synthetic}

\Cref{fig:app:syn_regret_curves_full} provides cumulative regret trajectories across all
synthetic scenarios described in \Cref{app:dgp}.

\begin{figure}[p]
  \centering
  \begin{subfigure}[t]{0.48\textwidth}
    \includegraphics[width=\textwidth]{regrets_curves_Friedman.pdf}
    \caption{Friedman}
  \end{subfigure}\hfill
  \begin{subfigure}[t]{0.48\textwidth}
    \includegraphics[width=\textwidth]{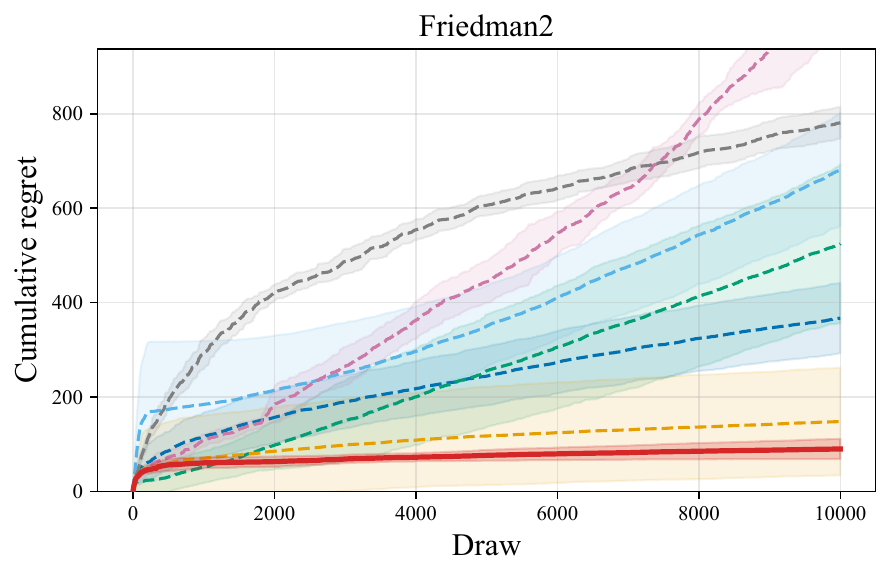}
    \caption{Friedman2}
  \end{subfigure}
  \\[1ex]
  \begin{subfigure}[t]{0.48\textwidth}
    \includegraphics[width=\textwidth]{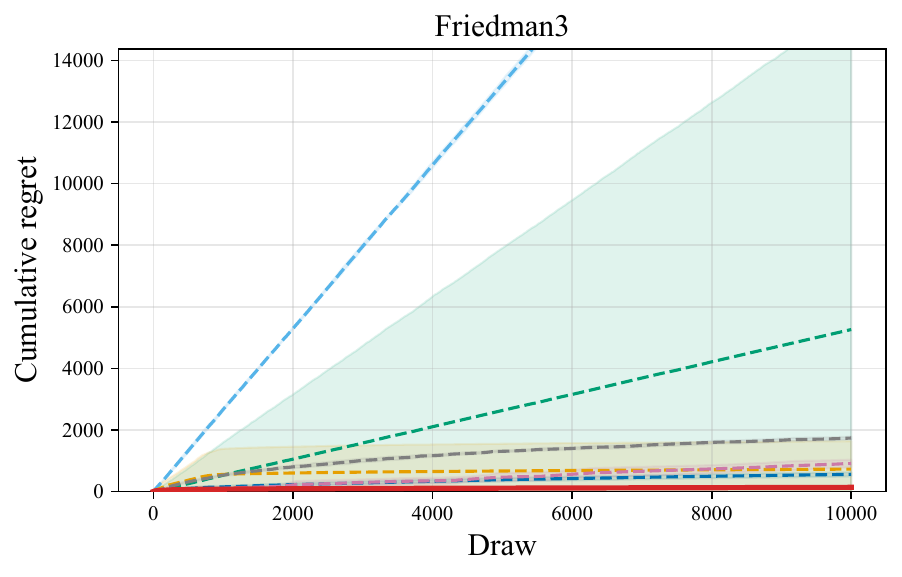}
    \caption{Friedman3}
  \end{subfigure}\hfill
  \begin{subfigure}[t]{0.48\textwidth}
    \includegraphics[width=\textwidth]{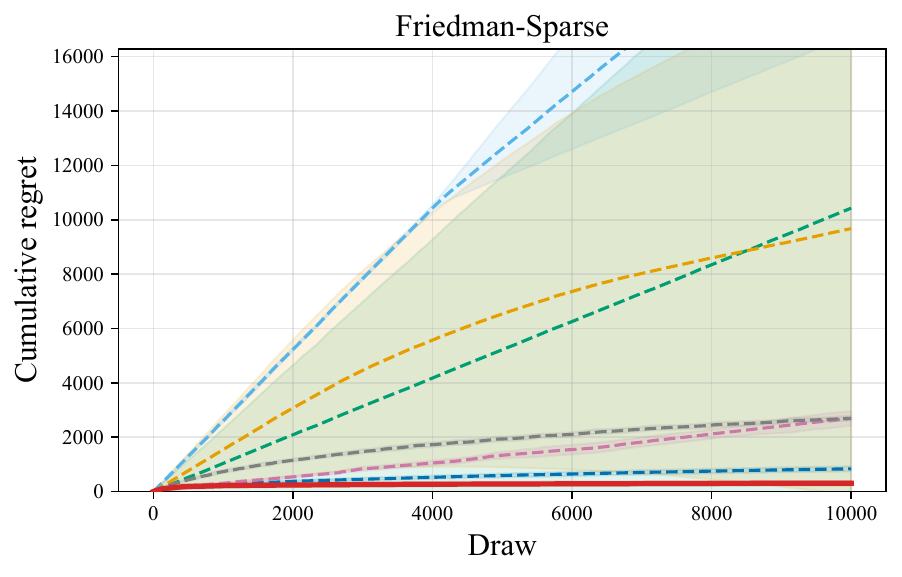}
    \caption{Friedman-Sparse}
  \end{subfigure}
  \\[1ex]
  \begin{subfigure}[t]{0.48\textwidth}
    \includegraphics[width=\textwidth]{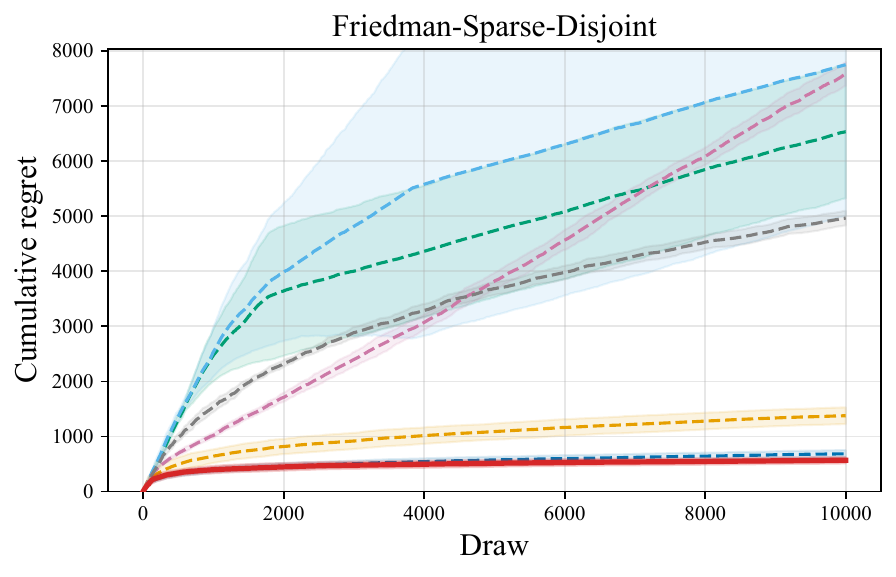}
    \caption{Friedman-Sparse-Disjoint}
  \end{subfigure}\hfill
  \begin{subfigure}[t]{0.48\textwidth}
    \includegraphics[width=\textwidth]{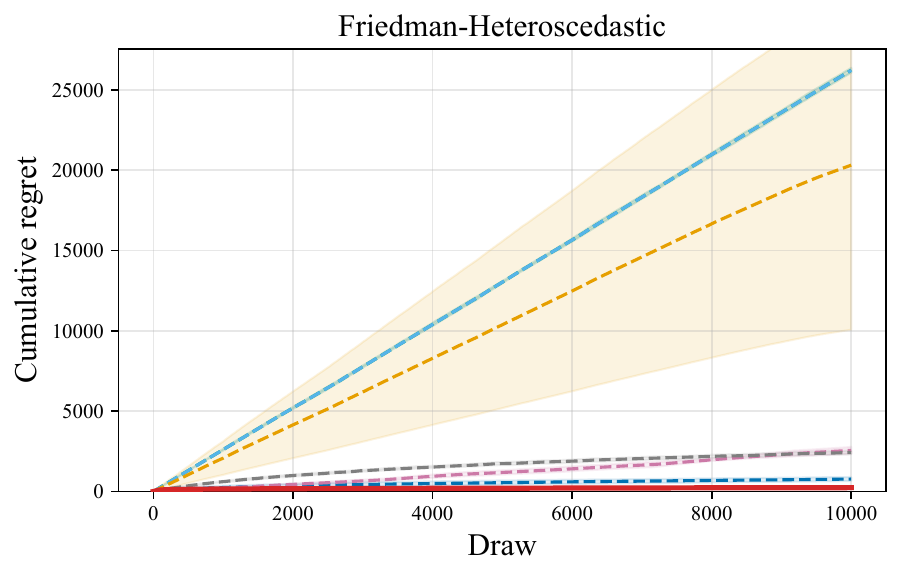}
    \caption{Friedman-Heteroscedastic}
  \end{subfigure}
  \\[1ex]
  \begin{subfigure}[t]{0.48\textwidth}
    \includegraphics[width=\textwidth]{regrets_curves_Linear.pdf}
    \caption{Linear}
  \end{subfigure}\hfill
  \begin{subfigure}[t]{0.48\textwidth}
    \includegraphics[width=\textwidth]{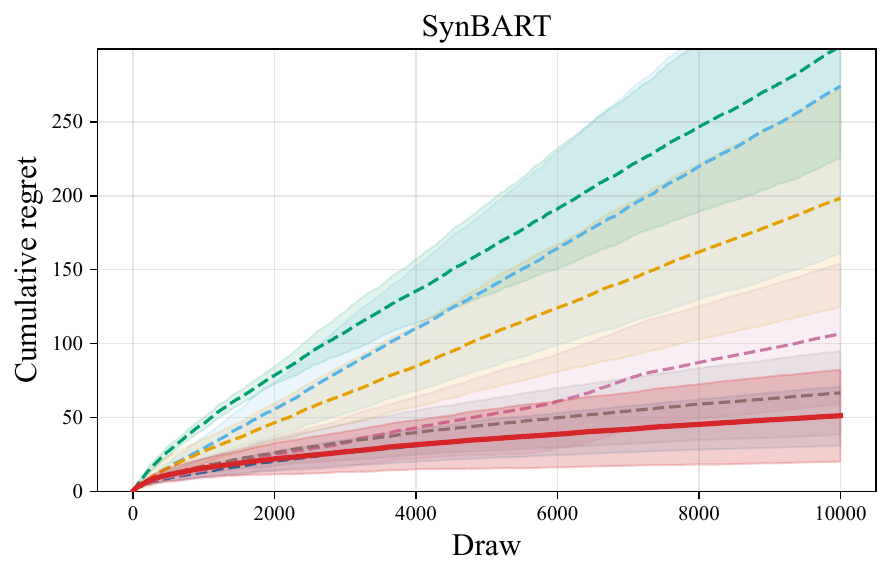}
    \caption{SynBART}
  \end{subfigure}
  \caption{Cumulative regret trajectories across all synthetic scenarios
  (mean $\pm$ SD, $R = 5$ replications, $T = 10{,}000$).}
  \label{fig:app:syn_regret_curves_full}
\end{figure}

\subsection{SubCLT Coverage Diagnostics}
\label{app:calibration}

To isolate the quality of the SubCLT approximation from adaptive data collection, we run
offline regression diagnostics on the Linear and Friedman DGPs. The DGPs match those in
\Cref{app:dgp}, but the bandit structure is removed: we treat each arm in isolation and
evaluate coverage on arm~0 only (Linear: $P = 10$, $\sigma^2 = 1$; Friedman: $P = 5$,
Friedman1 variant with shared correlation, $\sigma^2 = 1$). For each DGP, training size
$n \in \{16, 64, 256, 1024\}$, and replication, we independently resample the DGP
parameters, training set, and $50$ held-out query points. Coverage is evaluated for the
latent mean reward $f_0(x)$, not for a noisy reward draw.

TabICL-SubCLT forms nominal 95\% intervals from the predictive-sequence approximation
$\mathcal{N}(m_{s_n}(x), \hat{V}_{s_n}(x)/s_n)$ with geometric base $b = 2$.
For comparison, BART intervals are computed from 500 posterior draws ($m = 100$ trees, $n_{\mathrm{post}} = 500$, $n_{\mathrm{burn}} = 500$, $\alpha = 0.95$, $\beta = 2.0$)---note that we use the default BART prior of Chipman et al. with their recommended settings, which differs from the BFTS choices in \Cref{app:hyperparams}.
\Cref{fig:clt_frontier} reports empirical coverage and mean interval length, averaged
over 10 replications and 50 query points per replication.

\subsection{Full OpenML Results}
\label{app:full_openml}

\Cref{fig:app:openml_regret_curves} provides cumulative regret trajectories across all eight
OpenML datasets.

\begin{figure}[p]
  \centering
  \begin{subfigure}[t]{0.48\textwidth}
    \includegraphics[width=\textwidth]{regrets_curves_OpenML-Adult.pdf}
    \caption{Adult}
  \end{subfigure}\hfill
  \begin{subfigure}[t]{0.48\textwidth}
    \includegraphics[width=\textwidth]{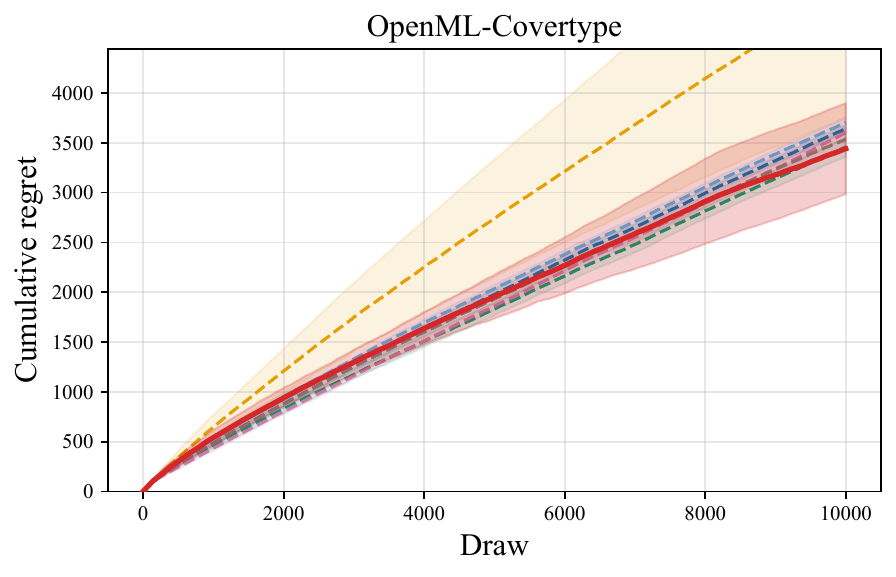}
    \caption{Covertype}
  \end{subfigure}
  \\[1ex]
  \begin{subfigure}[t]{0.48\textwidth}
    \includegraphics[width=\textwidth]{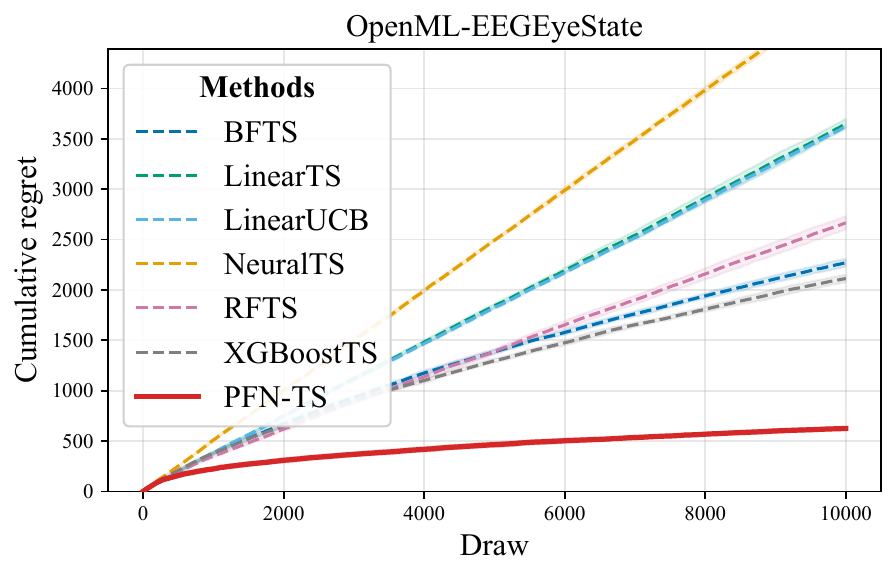}
    \caption{EEGEyeState}
  \end{subfigure}\hfill
  \begin{subfigure}[t]{0.48\textwidth}
    \includegraphics[width=\textwidth]{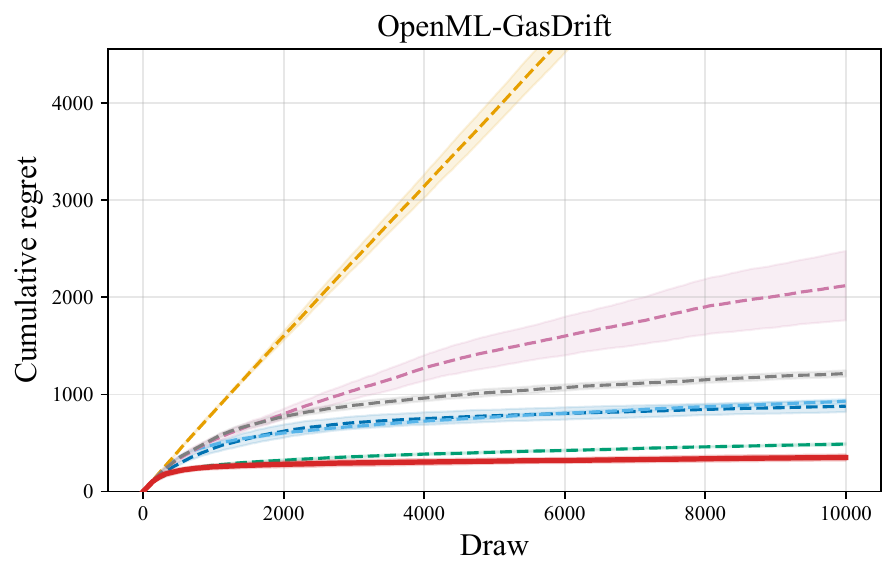}
    \caption{GasDrift}
  \end{subfigure}
  \\[1ex]
  \begin{subfigure}[t]{0.48\textwidth}
    \includegraphics[width=\textwidth]{regrets_curves_OpenML-MagicTelescope.pdf}
    \caption{MagicTelescope}
  \end{subfigure}\hfill
  \begin{subfigure}[t]{0.48\textwidth}
    \includegraphics[width=\textwidth]{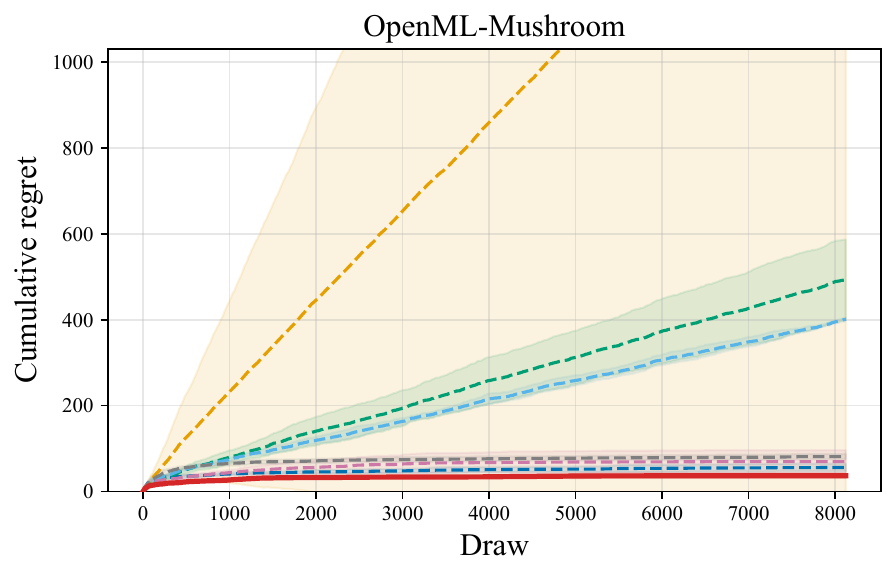}
    \caption{Mushroom}
  \end{subfigure}
  \\[1ex]
  \begin{subfigure}[t]{0.48\textwidth}
    \includegraphics[width=\textwidth]{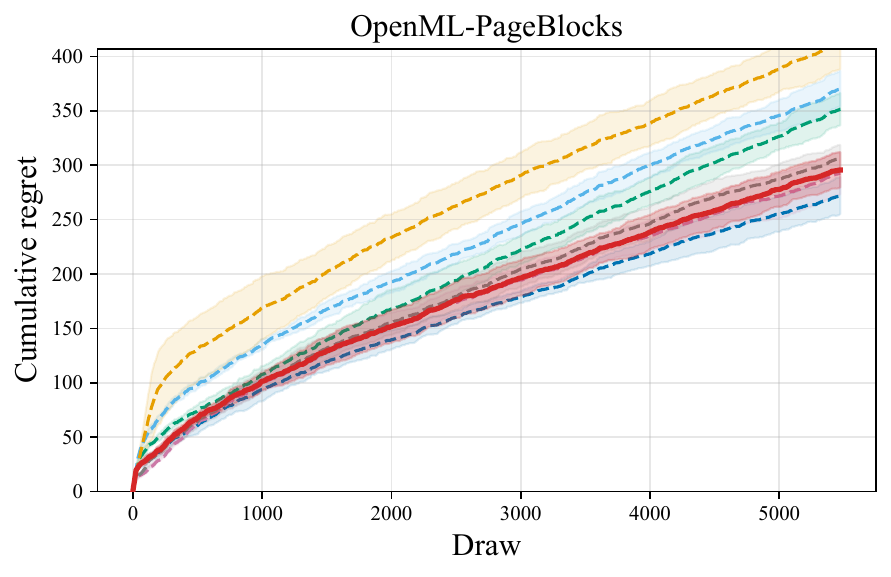}
    \caption{PageBlocks}
  \end{subfigure}\hfill
  \begin{subfigure}[t]{0.48\textwidth}
    \includegraphics[width=\textwidth]{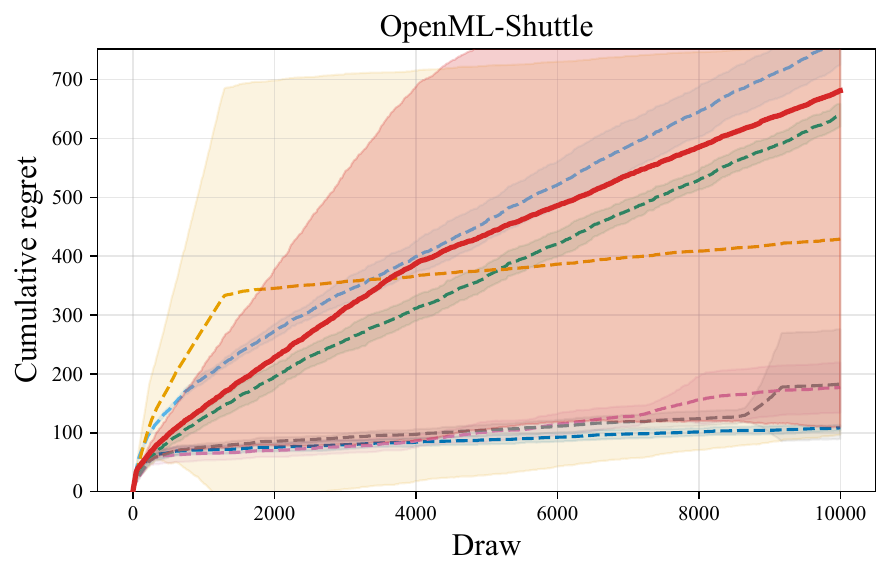}
    \caption{Shuttle}
  \end{subfigure}
  \caption{Cumulative regret trajectories on all eight OpenML datasets
  (mean $\pm$ SD, $R = 5$ replications).}
  \label{fig:app:openml_regret_curves}
\end{figure}

\subsection{OPE Diagnostics}
\label{app:ope_diagnostics}

\paragraph{Importance weight diagnostics.}
Under the static logging propensities $(0.4, 0.3, 0.3)$, importance weights are bounded by
$1 / \min_a \pi_0(a) = 10/3$. \Cref{fig:app:drinkless_weights} (left) shows the empirical
weight distribution; weights remain close to their maximum, confirming that positivity
violations are not a concern and SNIPS variance is well controlled.

\paragraph{Doubly-robust estimates.}
For each bootstrap replicate, a per-arm ridge outcome model $\hat{q}(x, a)$ is fit via
2-fold cross-fitting and used to form the DR estimate
\[
  \widehat{V}_{\mathrm{DR}}(\pi)
    = \frac{1}{n}\sum_{i=1}^n \Bigl[
        \sum_{a} \pi(a \mid x_i)\,\hat{q}(x_i, a)
        + w_i \bigl(r_i - \hat{q}(x_i, a_i)\bigr)
      \Bigr].
\]
\Cref{fig:app:drinkless_weights} (right) shows the DR policy value curve over the horizon;
the pattern closely tracks the SNIPS estimates, confirming robustness to the choice of
estimator.

\begin{figure}[tb]
  \centering
  \begin{subfigure}[t]{0.32\textwidth}
    \includegraphics[width=\textwidth]{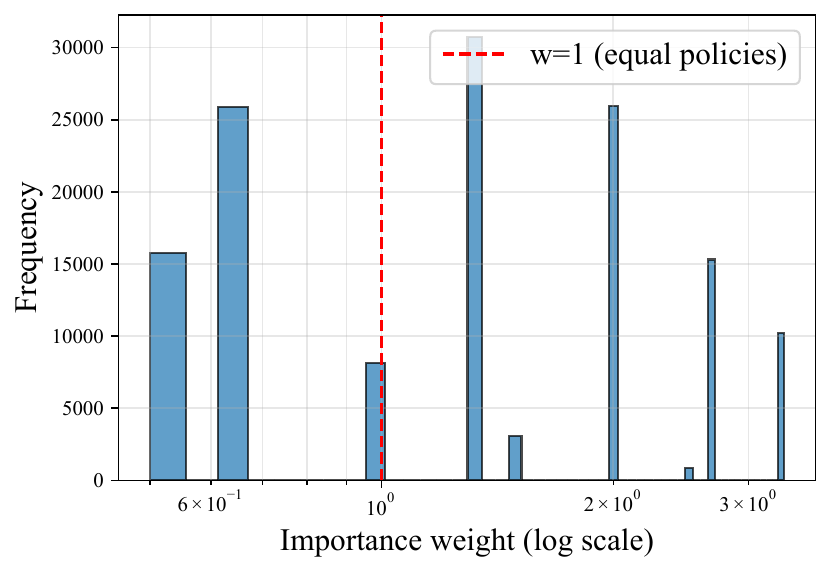}
    \caption{PFN-TS importance weights}
  \end{subfigure}\hfill
  \begin{subfigure}[t]{0.32\textwidth}
    \includegraphics[width=\textwidth]{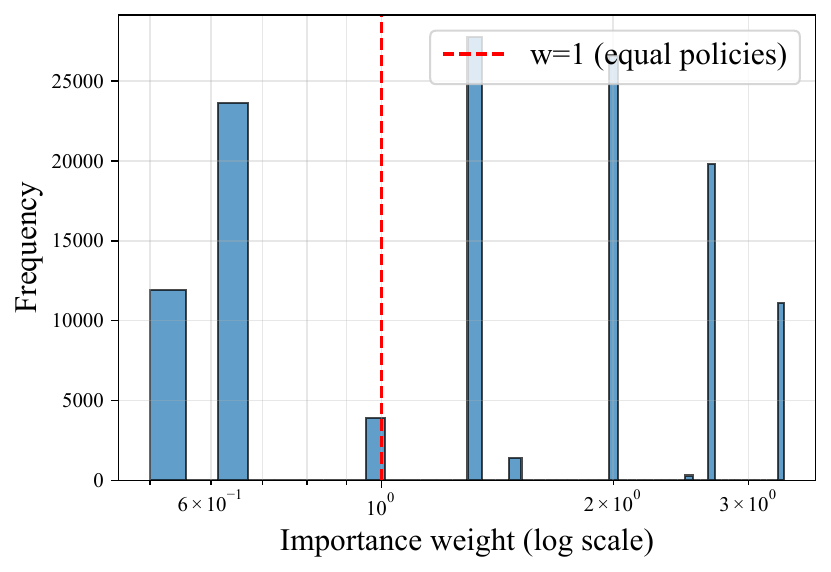}
    \caption{BFTS importance weights}
  \end{subfigure}\hfill
  \begin{subfigure}[t]{0.32\textwidth}
    \includegraphics[width=\textwidth]{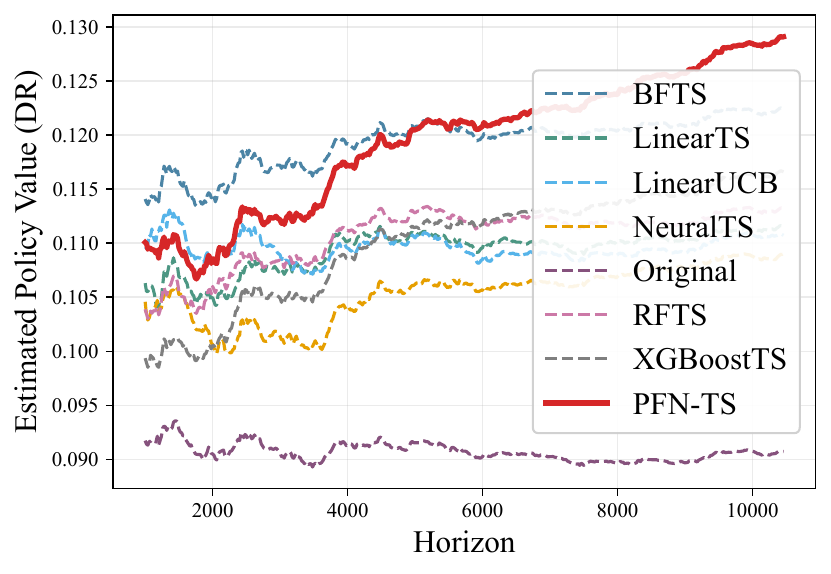}
    \caption{DR policy value}
  \end{subfigure}
  \caption{Drink Less OPE diagnostics. Left/centre: empirical importance weight
  distributions for PFN-TS and BFTS (log scale); weights are bounded by $10/3$ under
  the logging propensities, confirming no positivity violations. Right: policy value
  estimated via the doubly-robust estimator; the trajectory closely tracks the SNIPS
  estimates in \Cref{fig:ope}, confirming robustness to the choice of OPE estimator.}
  \label{fig:app:drinkless_weights}
\end{figure}

\begin{figure}[tb]
  \centering
  \begin{subfigure}[t]{0.48\textwidth}
    \includegraphics[width=\textwidth]{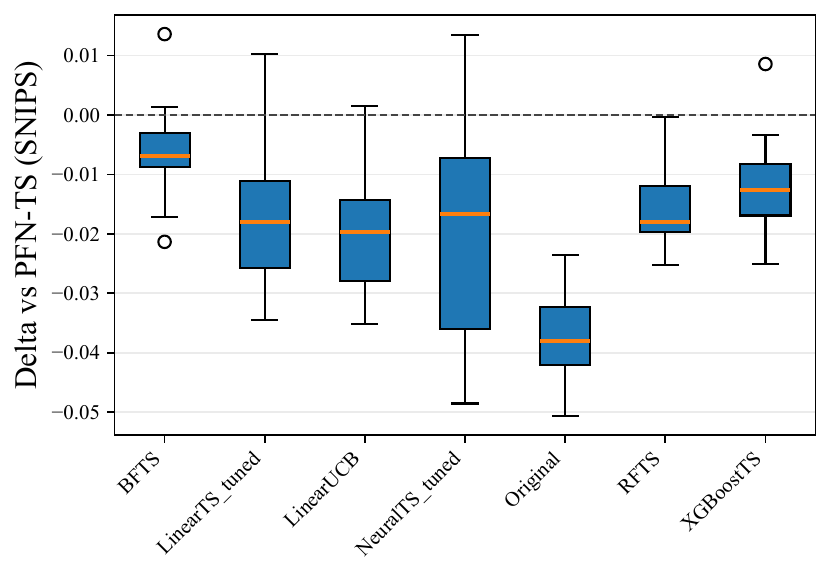}
    \caption{Final policy-value gap (SNIPS)}
  \end{subfigure}\hfill
  \begin{subfigure}[t]{0.48\textwidth}
    \includegraphics[width=\textwidth]{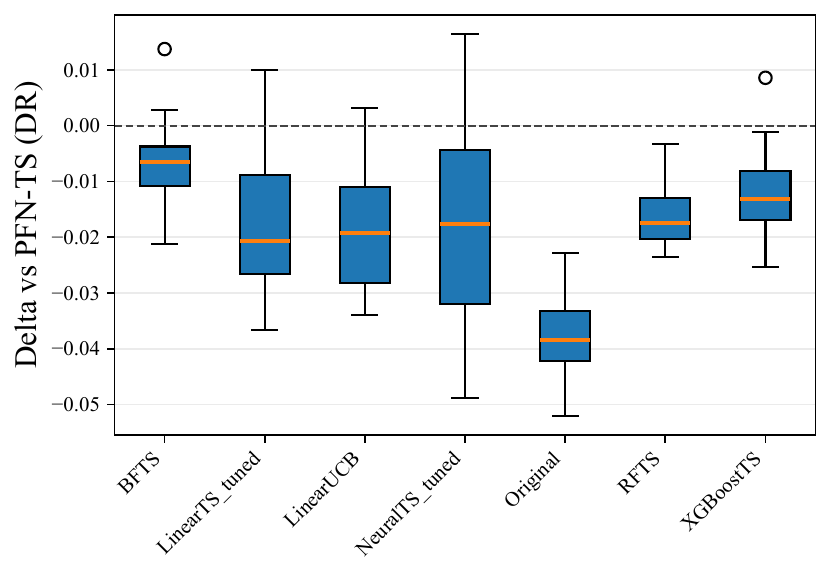}
    \caption{Final policy-value gap (DR)}
  \end{subfigure}
  \\[1ex]
  \begin{subfigure}[t]{0.48\textwidth}
    \includegraphics[width=\textwidth]{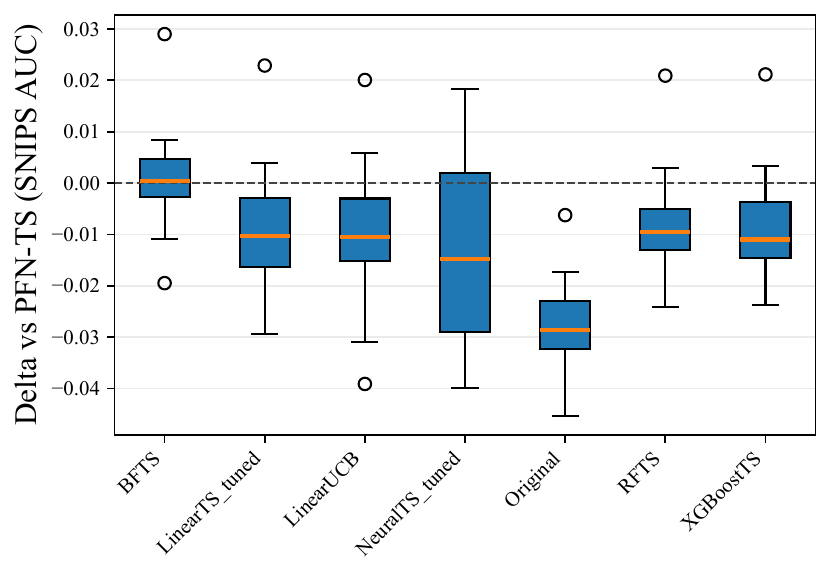}
    \caption{AUC-mean policy-value gap (SNIPS)}
  \end{subfigure}\hfill
  \begin{subfigure}[t]{0.48\textwidth}
    \includegraphics[width=\textwidth]{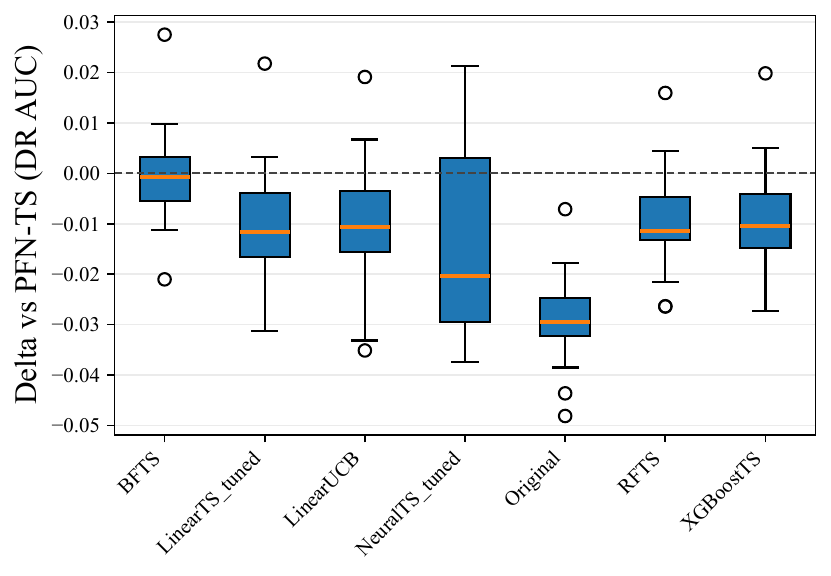}
    \caption{AUC-mean policy-value gap (DR)}
  \end{subfigure}
  \caption{Bootstrap distributions ($B = 30$ user-level cluster replicates) of the
  policy-value gap between each method and the best-performing baseline, on the
  Drink Less trial. Top: gap at the final horizon $T = 10{,}470$. Bottom: gap
  integrated over the full horizon (AUC mean). Positive values indicate that the
  method outperforms the best baseline.}
  \label{fig:app:drinkless_ope_gap_boxplots}
\end{figure}

\subsection{Ablation Plots}
\label{app:ablations}

\paragraph{Encoding strategy.}
\Cref{fig:app:ablation_encoding} compares PFN-TS (adaptive), PFN-TS-Disjoint (fixed
separate model), and PFN-TS-Joint (arm index concatenated to context) across synthetic and
OpenML scenarios.

\begin{figure}[tb]
  \centering
  \begin{subfigure}[t]{0.48\textwidth}
    \includegraphics[width=\textwidth]{ablation-encoding/regrets_curves_Friedman.pdf}
    \caption{Friedman}
  \end{subfigure}\hfill
  \begin{subfigure}[t]{0.48\textwidth}
    \includegraphics[width=\textwidth]{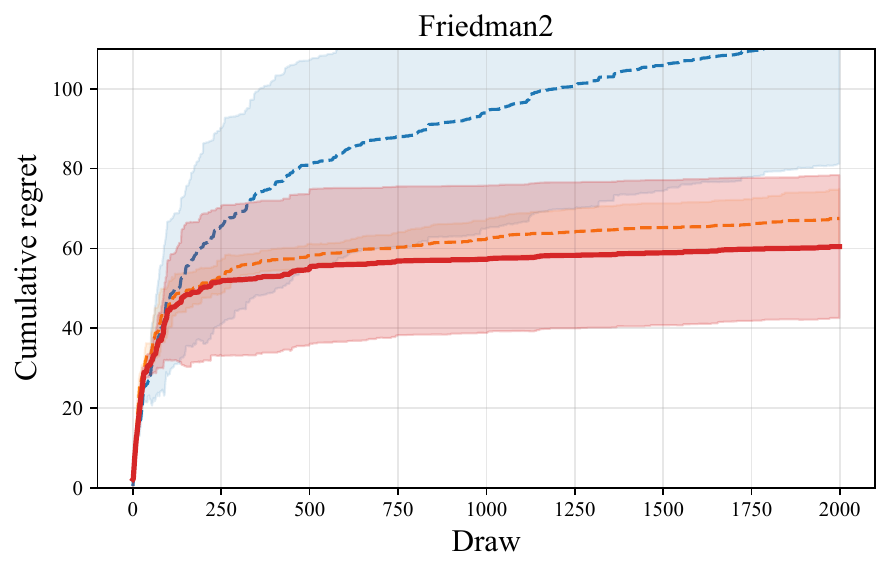}
    \caption{Friedman2}
  \end{subfigure}
  \\[1ex]
  \begin{subfigure}[t]{0.48\textwidth}
    \includegraphics[width=\textwidth]{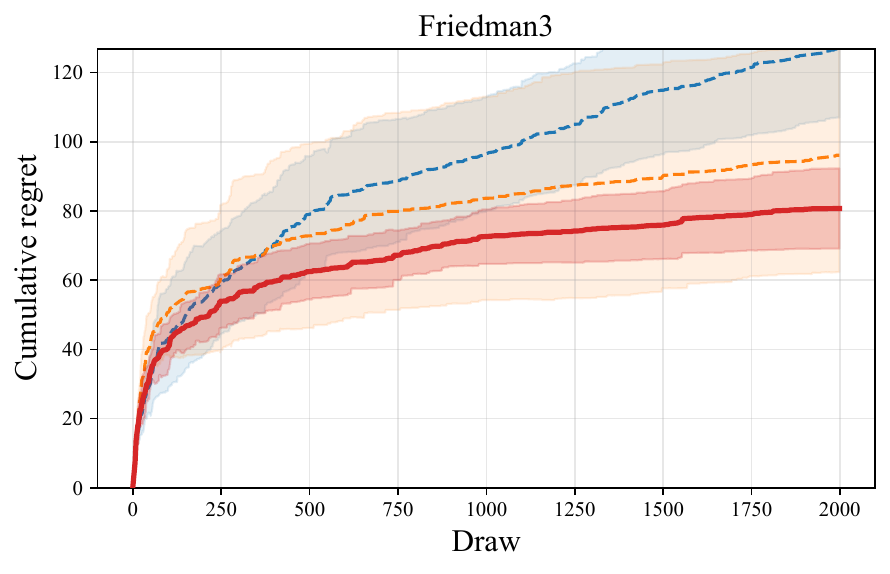}
    \caption{Friedman3}
  \end{subfigure}\hfill
  \begin{subfigure}[t]{0.48\textwidth}
    \includegraphics[width=\textwidth]{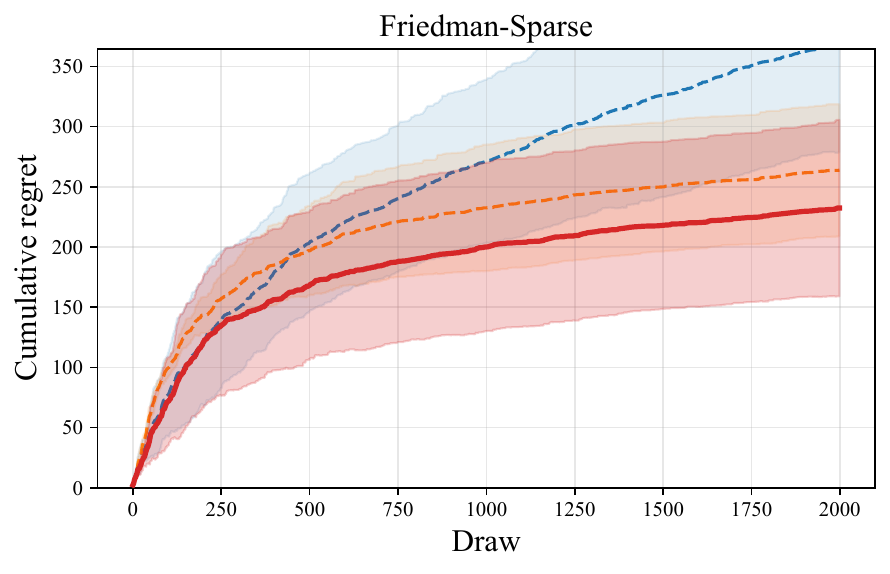}
    \caption{Friedman-Sparse}
  \end{subfigure}
  \\[1ex]
  \begin{subfigure}[t]{0.48\textwidth}
    \includegraphics[width=\textwidth]{ablation-encoding/regrets_curves_Friedman-Sparse-Disjoint.pdf}
    \caption{Friedman-Sparse-Disjoint}
  \end{subfigure}\hfill
  \begin{subfigure}[t]{0.48\textwidth}
    \includegraphics[width=\textwidth]{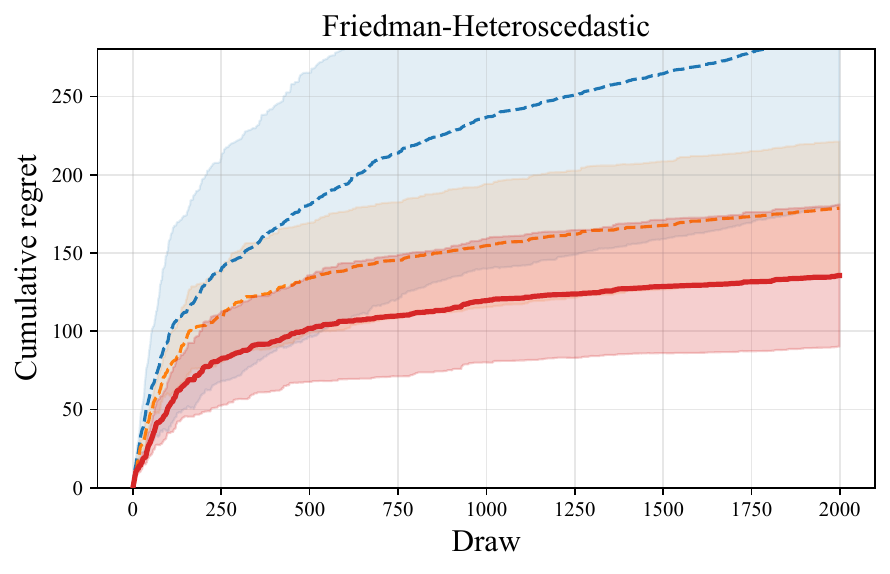}
    \caption{Friedman-Heteroscedastic}
  \end{subfigure}
  \\[1ex]
  \begin{subfigure}[t]{0.48\textwidth}
    \includegraphics[width=\textwidth]{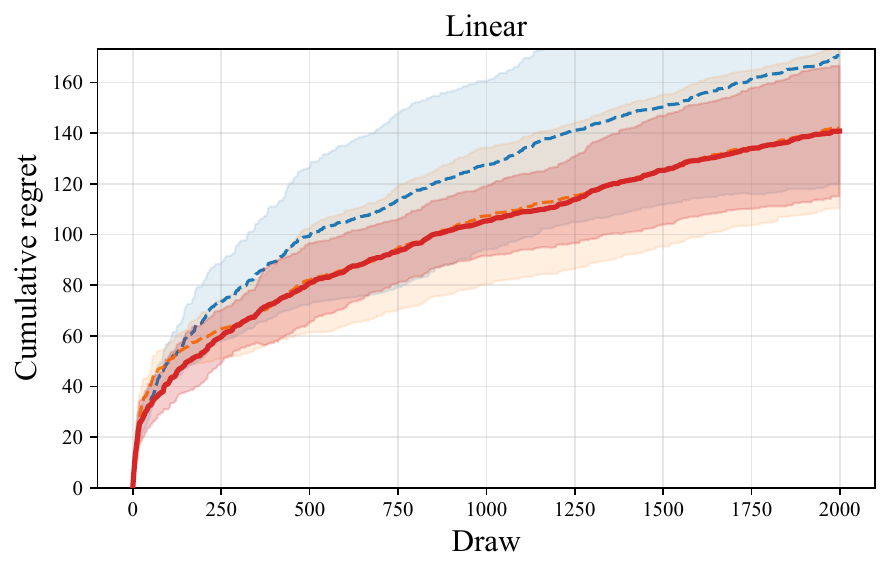}
    \caption{Linear}
  \end{subfigure}\hfill
  \begin{subfigure}[t]{0.48\textwidth}
    \includegraphics[width=\textwidth]{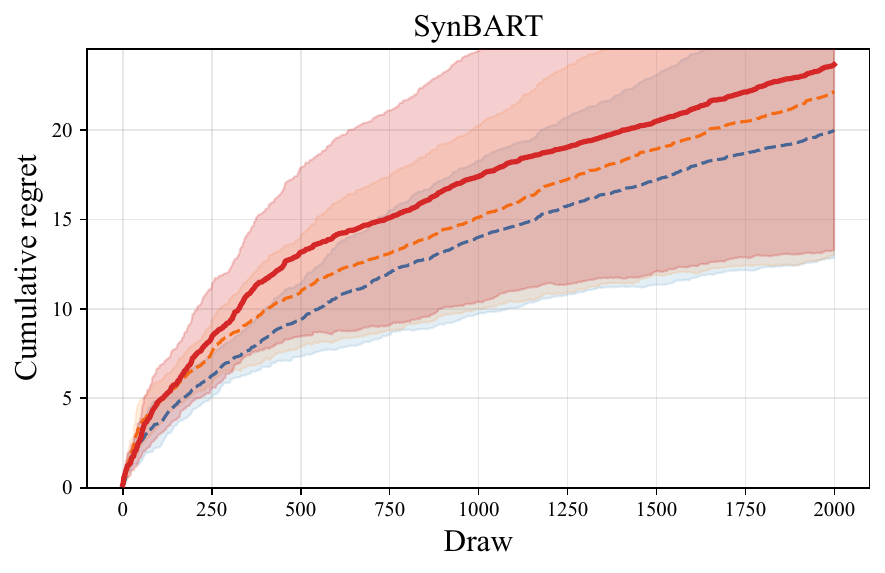}
    \caption{SynBART}
  \end{subfigure}
  \caption{Encoding ablation: cumulative regret trajectories comparing the adaptive
  encoding rule against \textit{Disjoint} (separate context per arm) and
  \textit{Joint} (arm index appended to context) across all synthetic scenarios
  (mean $\pm$ SD, $R = 5$ replications).}
  \label{fig:app:ablation_encoding}
\end{figure}

\paragraph{Subsampling grid.}
\Cref{fig:app:ablation_subsampling} compares PFN-TS with geometric base
$b \in \{1.5, 2, 3\}$ across the synthetic scenarios.

\begin{figure}[tb]
  \centering
  \begin{subfigure}[t]{0.48\textwidth}
    \includegraphics[width=\textwidth]{ablation-subsampling/regrets_curves_Friedman.pdf}
    \caption{Friedman}
  \end{subfigure}\hfill
  \begin{subfigure}[t]{0.48\textwidth}
    \includegraphics[width=\textwidth]{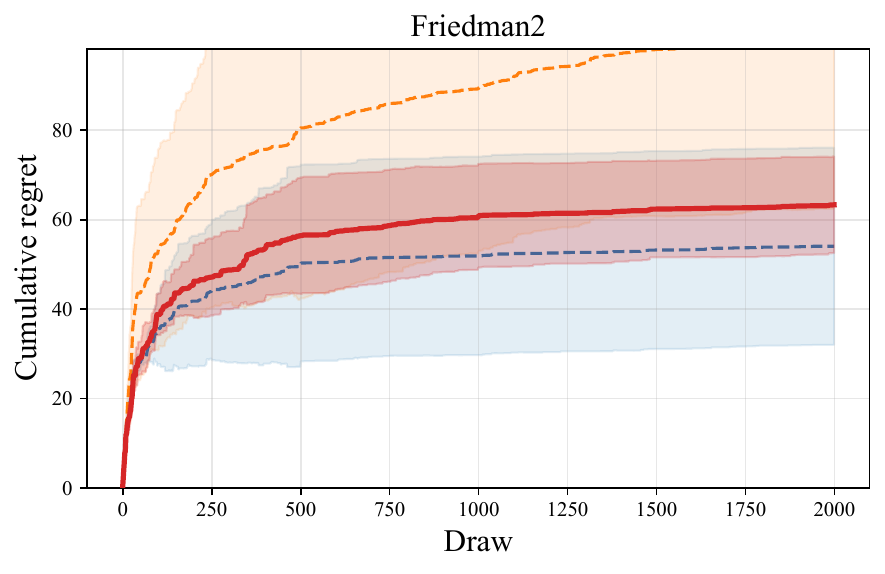}
    \caption{Friedman2}
  \end{subfigure}
  \\[1ex]
  \begin{subfigure}[t]{0.48\textwidth}
    \includegraphics[width=\textwidth]{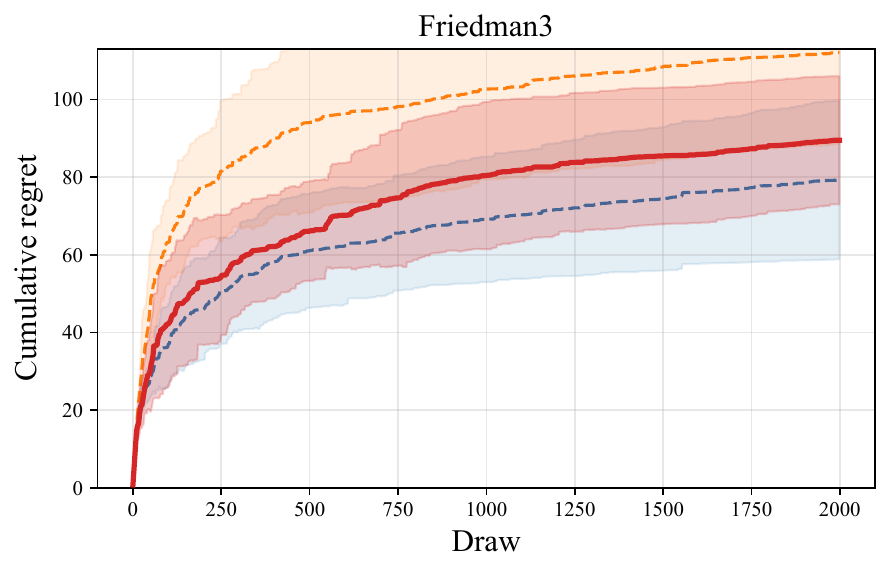}
    \caption{Friedman3}
  \end{subfigure}\hfill
  \begin{subfigure}[t]{0.48\textwidth}
    \includegraphics[width=\textwidth]{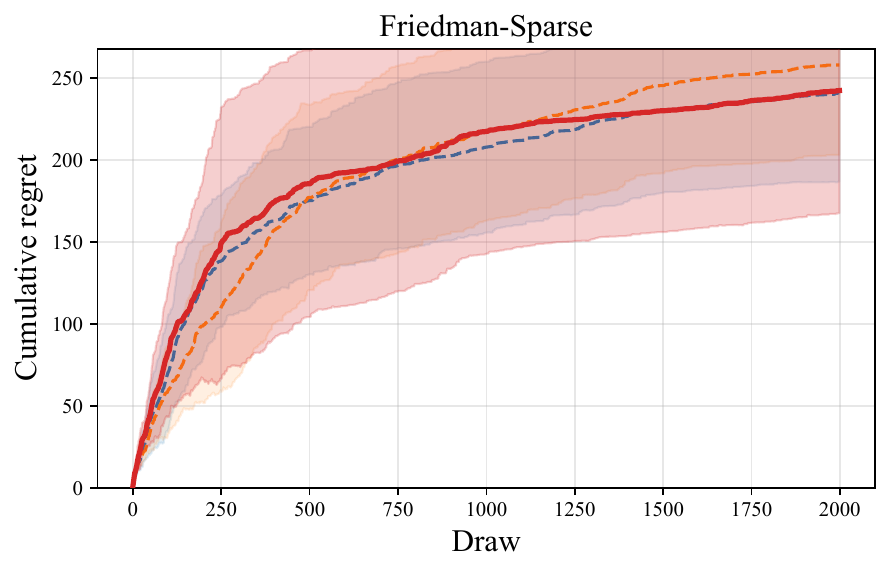}
    \caption{Friedman-Sparse}
  \end{subfigure}
  \\[1ex]
  \begin{subfigure}[t]{0.48\textwidth}
    \includegraphics[width=\textwidth]{ablation-subsampling/regrets_curves_Friedman-Sparse-Disjoint.pdf}
    \caption{Friedman-Sparse-Disjoint}
  \end{subfigure}\hfill
  \begin{subfigure}[t]{0.48\textwidth}
    \includegraphics[width=\textwidth]{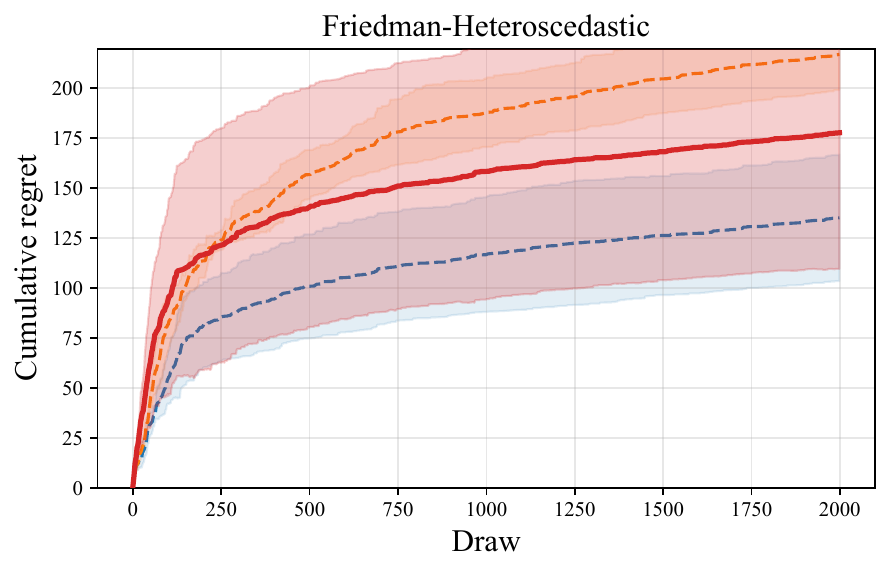}
    \caption{Friedman-Heteroscedastic}
  \end{subfigure}
  \\[1ex]
  \begin{subfigure}[t]{0.48\textwidth}
    \includegraphics[width=\textwidth]{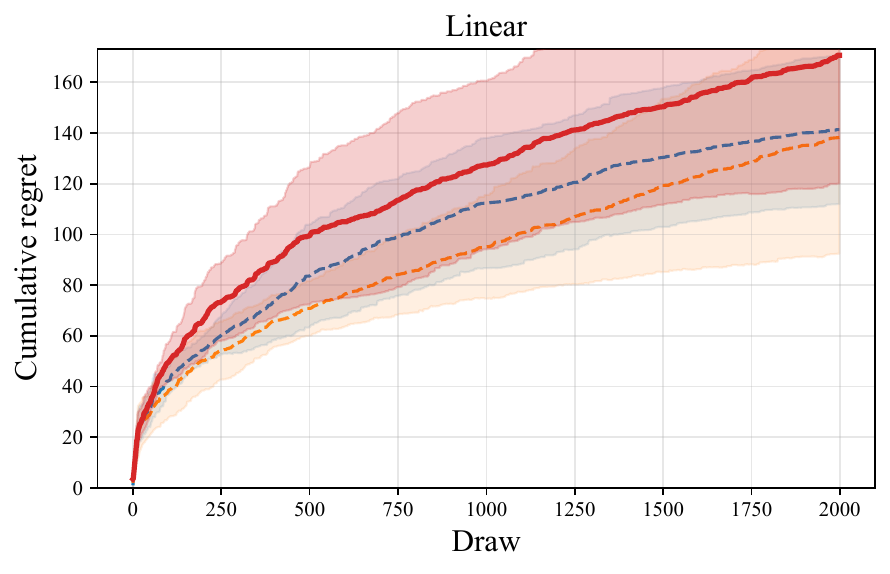}
    \caption{Linear}
  \end{subfigure}\hfill
  \begin{subfigure}[t]{0.48\textwidth}
    \includegraphics[width=\textwidth]{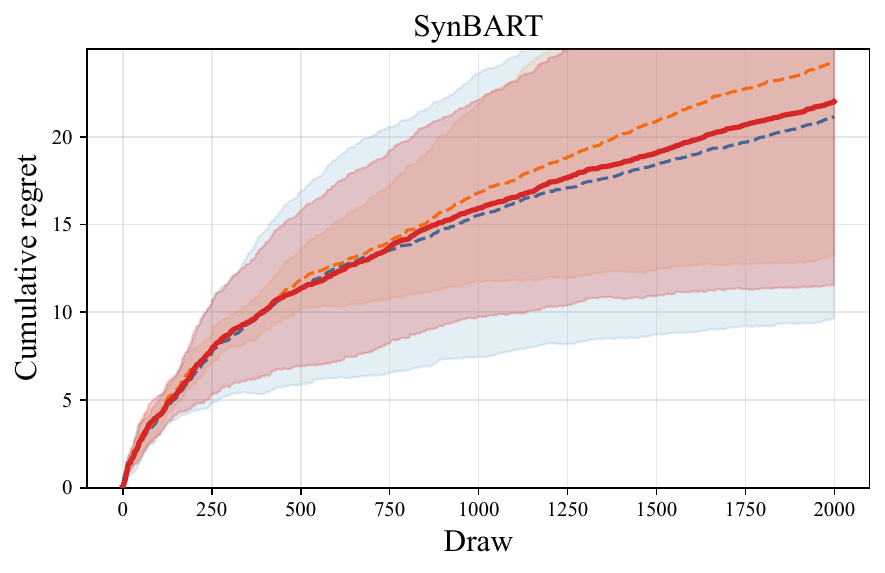}
    \caption{SynBART}
  \end{subfigure}
  \caption{Subsampling grid ablation: cumulative regret trajectories for geometric
  base $b \in \{1.5, 2, 3\}$ across all synthetic scenarios
  (mean $\pm$ SD, $R = 5$ replications).}
  \label{fig:app:ablation_subsampling}
\end{figure}

\paragraph{Decision rule.}
\Cref{fig:app:ablation_decision} compares Thompson sampling (PFN-TS), posterior predictive
sampling (PFN-PS), and a greedy baseline (highest predictive mean only).

\begin{figure}[p]
  \centering
  \begin{subfigure}[t]{0.48\textwidth}
    \includegraphics[width=\textwidth]{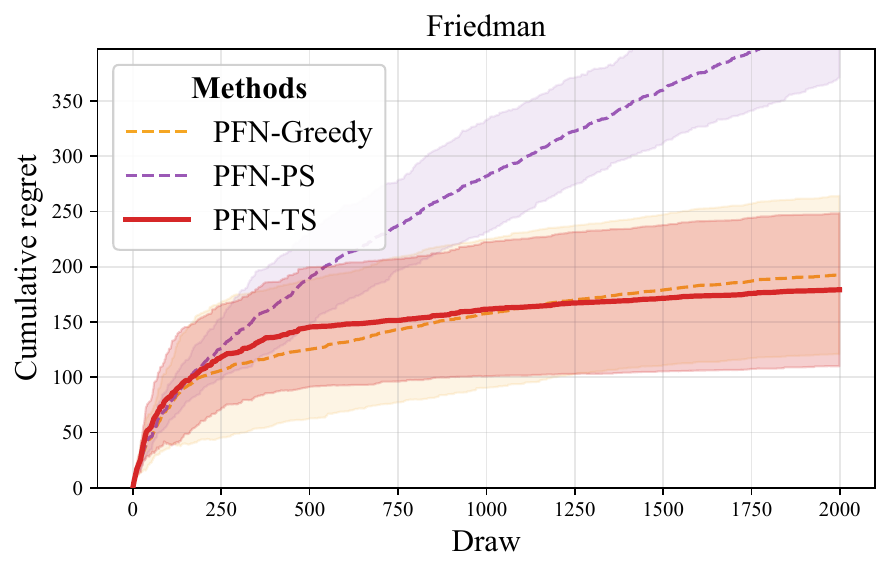}
    \caption{Friedman}
  \end{subfigure}\hfill
  \begin{subfigure}[t]{0.48\textwidth}
    \includegraphics[width=\textwidth]{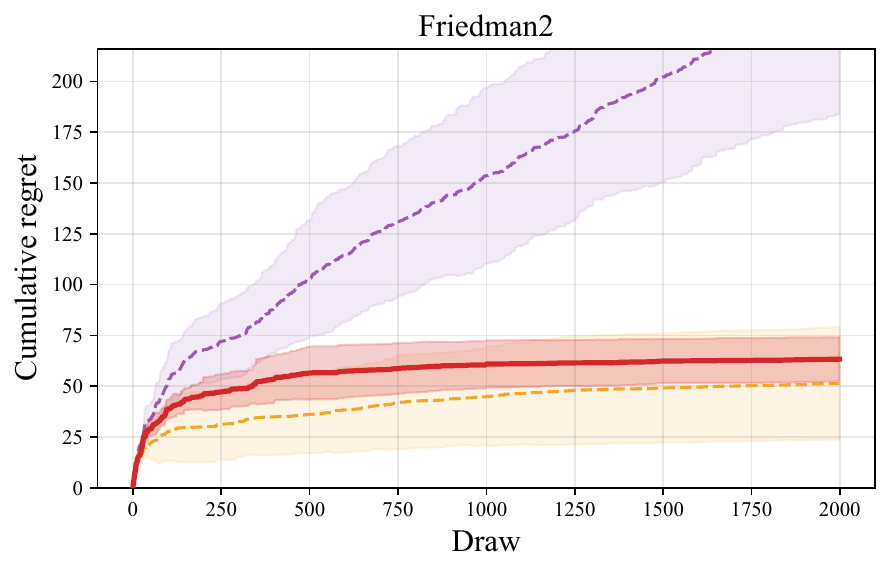}
    \caption{Friedman2}
  \end{subfigure}
  \\[1ex]
  \begin{subfigure}[t]{0.48\textwidth}
    \includegraphics[width=\textwidth]{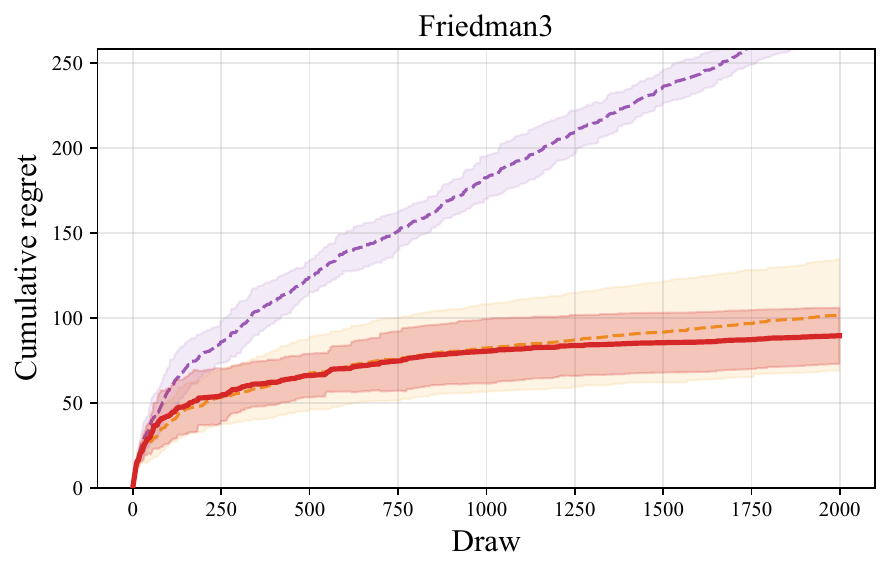}
    \caption{Friedman3}
  \end{subfigure}\hfill
  \begin{subfigure}[t]{0.48\textwidth}
    \includegraphics[width=\textwidth]{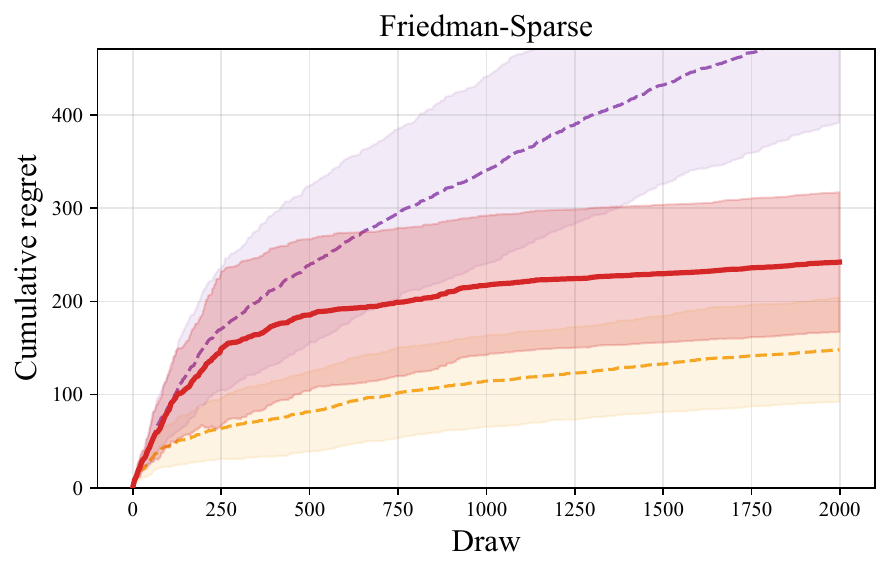}
    \caption{Friedman-Sparse}
  \end{subfigure}
  \\[1ex]
  \begin{subfigure}[t]{0.48\textwidth}
    \includegraphics[width=\textwidth]{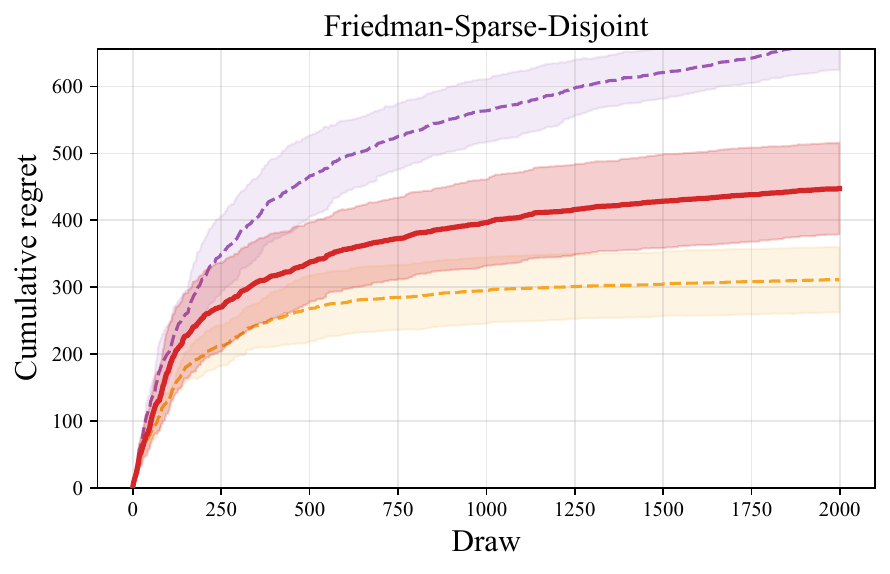}
    \caption{Friedman-Sparse-Disjoint}
  \end{subfigure}\hfill
  \begin{subfigure}[t]{0.48\textwidth}
    \includegraphics[width=\textwidth]{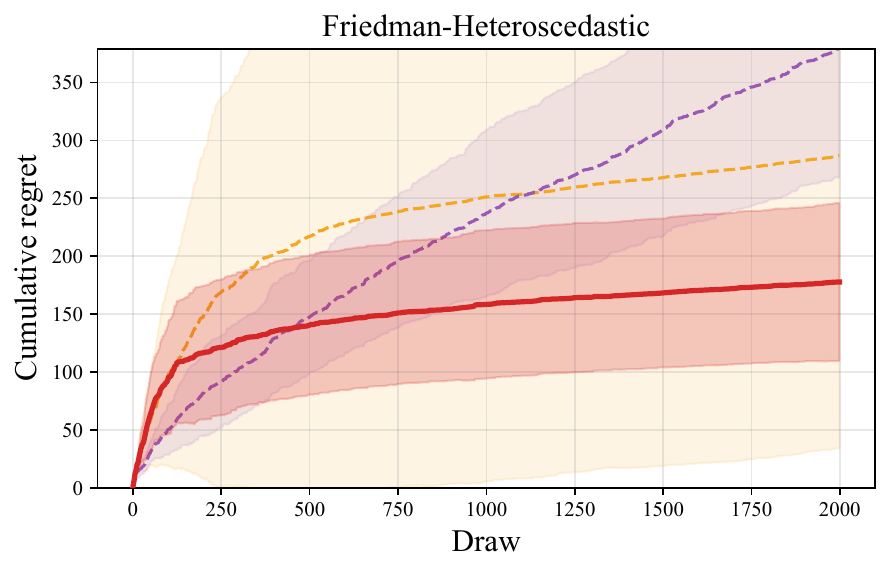}
    \caption{Friedman-Heteroscedastic}
  \end{subfigure}
  \\[1ex]
  \begin{subfigure}[t]{0.48\textwidth}
    \includegraphics[width=\textwidth]{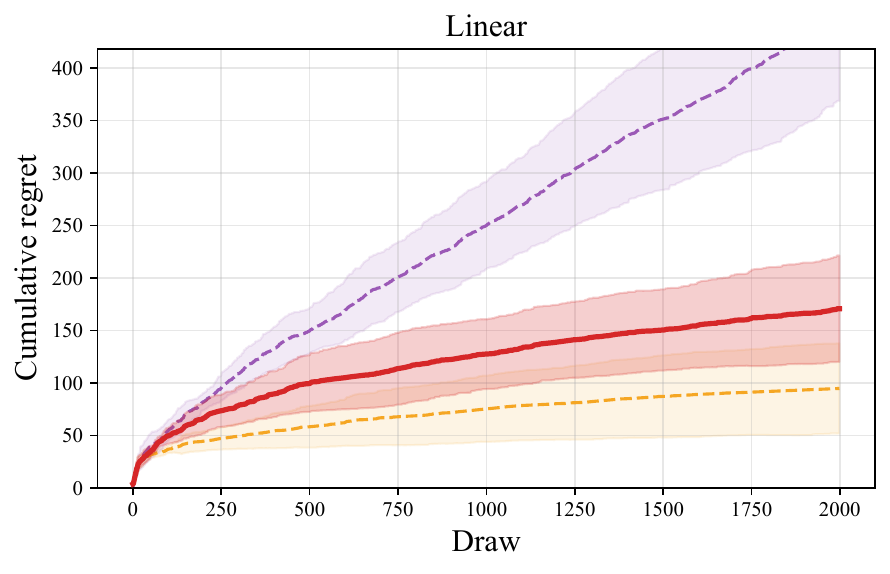}
    \caption{Linear}
  \end{subfigure}\hfill
  \begin{subfigure}[t]{0.48\textwidth}
    \includegraphics[width=\textwidth]{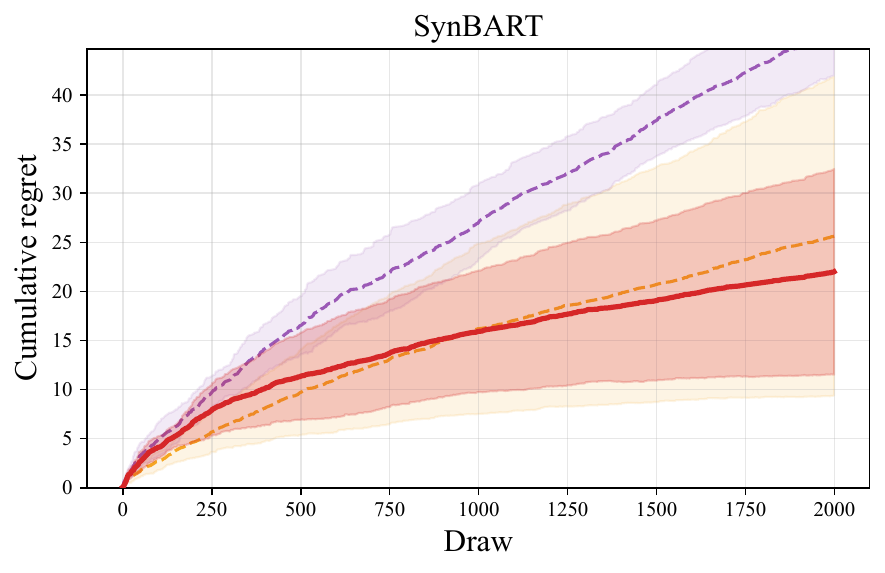}
    \caption{SynBART}
  \end{subfigure}
  \caption{Decision rule ablation: cumulative regret trajectories comparing Thompson
  sampling (PFN-TS), PFN-PS, and greedy across synthetic scenarios
  (mean $\pm$ SD, $R = 5$ replications). Continued in \Cref{fig:app:ablation_decision_openml}.}
  \label{fig:app:ablation_decision}
\end{figure}

\begin{figure}[p]
  \centering
  \begin{subfigure}[t]{0.48\textwidth}
    \includegraphics[width=\textwidth]{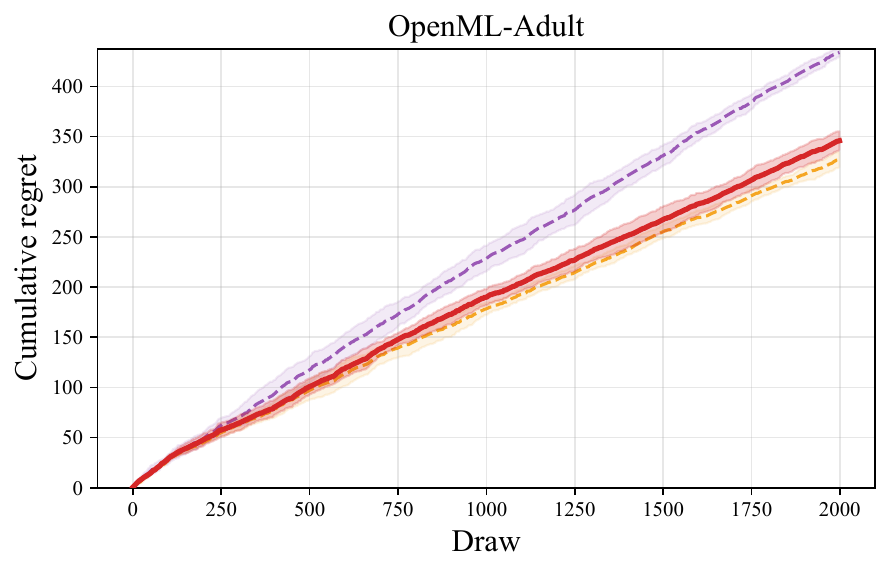}
    \caption{Adult}
  \end{subfigure}\hfill
  \begin{subfigure}[t]{0.48\textwidth}
    \includegraphics[width=\textwidth]{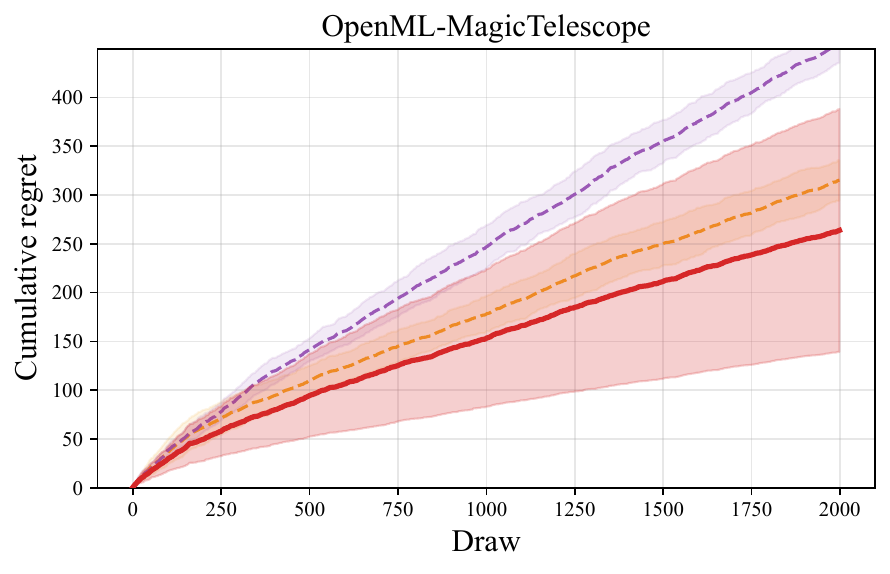}
    \caption{MagicTelescope}
  \end{subfigure}
  \\[1ex]
  \begin{subfigure}[t]{0.48\textwidth}
    \includegraphics[width=\textwidth]{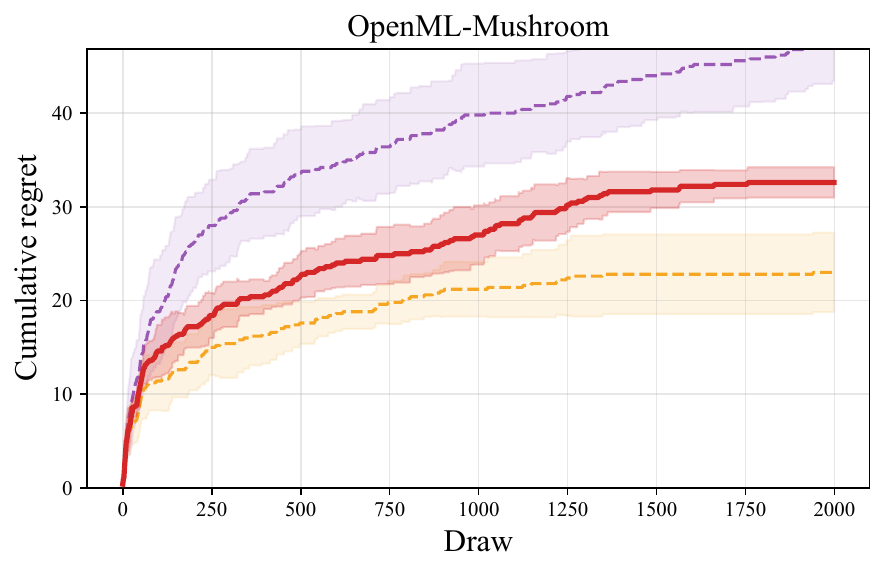}
    \caption{Mushroom}
  \end{subfigure}\hfill
  \begin{subfigure}[t]{0.48\textwidth}
    \includegraphics[width=\textwidth]{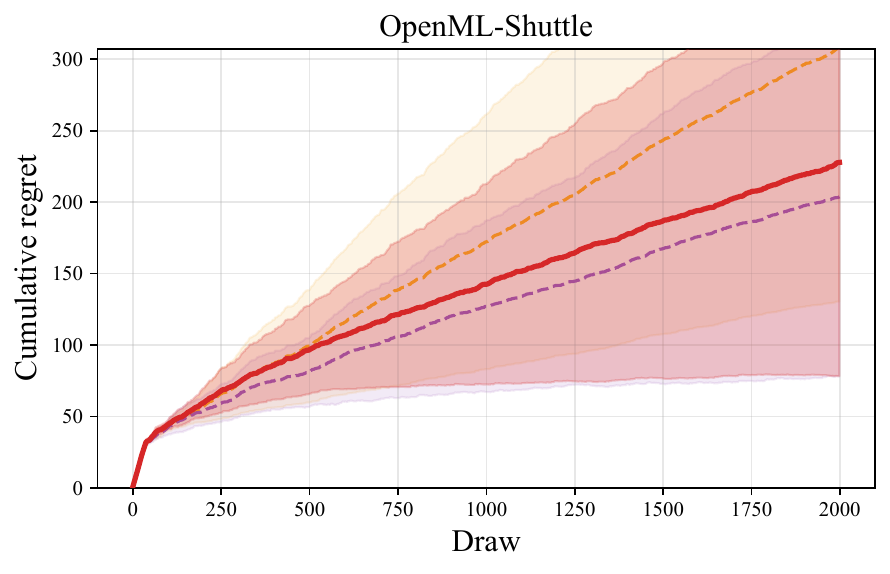}
    \caption{Shuttle}
  \end{subfigure}
  \caption{Decision rule ablation (OpenML scenarios, continued from
  \Cref{fig:app:ablation_decision}).}
  \label{fig:app:ablation_decision_openml}
\end{figure}







\end{document}